\documentclass[11]{amsart}

\newtheorem{theorem}{Theorem}[section]
\newtheorem{lemma}[theorem]{Lemma}

\theoremstyle{definition}
\newtheorem{definition}[theorem]{Definition}
\newtheorem{example}[theorem]{Example}
\newtheorem{xca}[theorem]{Exercise}
\newtheorem{corollary}[theorem]{Corollary}
\newtheorem{proposition}[theorem]{Proposition}
\newtheorem{assumption}[theorem]{Hypothesis}
\newtheorem{hypothesis}[theorem]{Hypothesis}

\theoremstyle{remark}
\newtheorem{remark}[theorem]{Remark}

\numberwithin{equation}{section}

\usepackage{amsmath}
  \usepackage{paralist}
  \usepackage{graphics} 
  \usepackage{epsfig} 
\usepackage{graphicx,appendix}  \usepackage{epstopdf}
 \usepackage[colorlinks=true]{hyperref}
\hypersetup{urlcolor=blue, citecolor=red}

\makeatletter \providecommand{\@LN}[2]{} \makeatother
\def\currentvolume{X}
 \def\currentissue{X}
  \def\currentyear{200X}
   \def\currentmonth{XX}
    \def\ppages{X--XX}

\theoremstyle{definition}

\if{\newtheorem{theorem}{Theorem}[section]
\newtheorem{lemma}[theorem]{Lemma}

\theoremstyle{definition}
\newtheorem{definition}[theorem]{Definition}

\theoremstyle{remark}
\newtheorem{remark}[theorem]{Remark}
}\fi
\numberwithin{equation}{section}

\usepackage{graphicx, dsfont}

\usepackage{amsmath,subfig,url}
\usepackage{paralist}
\usepackage{graphics} 
\usepackage{epsfig} 
\usepackage{graphicx}  \usepackage{epstopdf}

\usepackage[]{algorithm2e}
\usepackage{hhline}

 \usepackage[colorlinks=true]{hyperref}
\hypersetup{urlcolor=blue, citecolor=red}

\def\currentvolume{X}
 \def\currentissue{X}
  \def\currentyear{200X}
   \def\currentmonth{XX}
    \def\ppages{X--XX}

\newcommand{\norm}[2]{\left\lVert#1\right\rVert_{#2}}

\newcommand\finsquare{
\end{proof}\medskip } 

\numberwithin{equation}{section}

\def\weight(#1,#2){c_{#1,#2}}

\if {

}{{\caption}}
} \fi

\def\dh{\hat{d}}

\def\tt{\tilde{t}}

\def\calf{{\mathcal F}}

\def\calh{{\mathcal H}}

\def\calr{{\mathcal R}}

\def\D{\mathcal{D}}

\def\K{\mathcal{K}}
\def\L{\mathcal{L}}

\def\1B{{\bf  1}}

\def\1B{{\bf  1}}

\newcommand{\NN}{\mathbb{N}}

\newcommand\be{\begin{equation}}
\newcommand\ee{\end{equation}}
\newcommand\ba{\begin{array}}
\newcommand\ea{\end{array}}
\newcommand{\bea}{\begin{eqnarray}}
\newcommand{\eea}{\end{eqnarray}}
\newcommand{\bean}{\begin{eqnarray*}}
\newcommand{\eean}{\end{eqnarray*}}

\def\rar{\rightarrow}

\usepackage{pgfplotstable}
\usepackage{pgfplots}
\usepackage{xcolor}

\newcommand{\R}{\mathds{R}}

\newcommand{\bfunc}{\mathbf{b}}
\newcommand{\y}{\mathbf{y}}

\newcommand{\Lagr}{\mathcal{L}}
\newcommand{\ufunc}{\mathbf{u}}
\DeclareMathOperator{\ReLU}{ReLU}

\DeclareMathOperator{\logistic}{logit}
\DeclareMathOperator{\softplus}{softplus}

\usepackage{amsmath}
\usepackage{mathtools}
\usepackage{amsthm}
\usepackage{amsfonts}
\usepackage{dsfont}
\usepackage{float}
\usepackage{lipsum}
\usepackage{tikz}
\usetikzlibrary{shapes.geometric, arrows}
\usetikzlibrary{positioning, calc}
\usepackage{tikz-cd}
\usepackage{comment}
\usepackage[colorinlistoftodos]{todonotes}
\usepackage{multirow}
\usepackage{pgfplots}
\pgfplotsset{width=10cm,compat=1.9}
\usepackage{forest}
\usepackage{booktabs, makecell, tabularx} 
\usepackage{rotating}
\usepackage[export]{adjustbox}

\begin{document}

\title[Classification with Runge-Kutta networks]{Classification with Runge-Kutta networks and feature space augmentation}


\author{Elisa Giesecke}
\address{Institut f\"ur Mathematik\\
            Humboldt-Universit\"at zu Berlin\\
            10099 Berlin, Germany}
\curraddr{}
\email{giesekee@hu-berlin.de}
\thanks{}

\author{Axel Kr\"oner}
\address{Institut für Mathematik\\
Martin-Luther-Universit\"at Halle-Wittenberg\\
Theodor-Lieser-Str. 5\\
06120 Halle}
\curraddr{}
\email{axel.kroener@mathematik.uni-halle.de}
\thanks{The second author is supported by DAAD project 57570343.}

\subjclass[2010 Mathematics Subject Classification.]{65L06, 68T07.}

\keywords{deep learning, Runge-Kutta networks, augmented neural ODEs, point classification.}

\date{\today}

\dedicatory{}

\begin{abstract}
In this paper we combine an approach based on Runge-Kutta Nets considered in [\emph{Benning et al., J. Comput. Dynamics, 9, 2019}] and a technique on augmenting the input space in [\emph{Dupont et al., NeurIPS}, 2019] to obtain network architectures which show a better numerical performance for deep neural networks in point and image classification problems. The approach is illustrated with several examples implemented in PyTorch.
\end{abstract}

\maketitle

\tableofcontents

\tableofcontents

\section{Introduction}\label{sec:intro}

In the framework of deep learning, network architectures can be derived from discretization schemes for ordinary differential equations including higher order, implicit, and  adaptive schemes, see, e.\,g., Benning et al. \cite{benning2019DLasOCP}, Celledoni et al. \cite{MR4308177},  E~\cite{weinan2017MLviaDynamicalSystems}, Ruiz-Balet and Zuazua \cite{BaletZuazua:2021}; 
%
the corresponding differential equation taken as starting point for designing neural networks is also called \textit{neural ordinary differential equation (NODE)}, cf. \cite{chen2018NODE}.  In this paper we combine sympletic partitioned Runge-Kutta (RK) methods for designing networks as proposed in~\cite{benning2019DLasOCP} with augmenting the input space as considered in Dupont, Doucet, and Teh~\cite{dupont2019ANODE}. This leads to an improved performance in experiments on binary and multiclass classification on two and three dimensional point datasets. Furthermore, we show that the observations generalize to image classification. The code to reproduce all experiments of this paper is based on PyTorch developed by the first author, see \cite{ARKN} and is available at GitHub 
\begin{center}
\texttt{https://github.com/ElisaGiesecke/augmented-RK-Nets}. 
\end{center}

%
%
%
%
%
%
\if{
\subsection{Architectures and ordinary differential equations}
Differential equations have been studied thoroughly, whereas theory in deep learning has emerged recently. Consequently, ODEs are better understood and optimization tools further developed so that it is easier to handle continuous dynamical systems compared to discrete ones in form of DNNs. The knowledge gained about their properties can then be used in order to draw conclusions about the properties of the respective neural network that the ODE is derived from. }\fi
\if{Thus, we can shed light into the black-box NN model and explain why they behave in a certain way. Additionally, continuous dynamical systems can be easily modified by imposing some kind of structure on them or adding constraints, see \cite[p.\,2]{weinan2017MLviaDynamicalSystems}, which can then be transferred to their discrete network analog. Based on this analysis, a more sophisticated choice can be made regarding which network architecture is suitable for achieving the criteria specified in Section \ref{sec:motivation}. }\fi
%
\if{
This approach has been particularly helpful with respect to stability issues \cite{haber2017stability}, see also \cite[p.\,1]{benning2019DLasOCP}. The objective is to avoid unstable forward propagation of data through the neural network, meaning that small perturbations of the input data should not lead to large changes in the network output. This is crucial for obtaining robust predictions because noise can otherwise cause wrong predictions and the model might not be able to generalize well to new data. Since the stability of the network corresponds to the stability of its underlying ODE, we are interested in bounding the solution $\y$ of \eqref{IVP_NN} at time $T$ by the initial condition at time zero \cite[p.\,4]{benning2019DLasOCP}. On the basis of spectral theory, \cite[sec.\,3]{haber2017stability} develop stability criteria for the generalized ResNet architecture \eqref{Euler}, which include conditions on the weight matrices $K^{[l]}$ and on the step size $h$. Moreover, they propose to add these to the optimization problem as constraints. This leads not only to a stable forward propagation of the trained network, but it also tackles the stability-related problem of vanishing or exploding gradients which might occur during training and that lets the gradient descent algorithm fail. In conclusion, ensuring stability addresses two key aspects of network design, training and generalization.
}\fi
%
The approach is motivated by the following observation stated in \cite{dupont2019ANODE}: The flow of NODEs is not able to represent certain functions, since their trajectories cannot intersect each other. Only features that are homeomorphic to the input space can evolve within the continuous model. We illustrate this fact by the  \verb|donut_2D| dataset depicted in Figure \ref{fig:data_examples}. We observe that mappings that separate data points of this set according to their class do not preserve the topology of the two dimensional input space and can therefore not be learned by models based on continuous dynamical systems in two space dimensions. Although ResNets can actually approximate these mappings because the discrete trajectories may cross each other, the learned flows are then highly complex. That means the points in the center of the donut would need to be squeezed through the gaps between the points on the ring as in Figure \ref{fig:width_donut_2D} -- a transformation which is difficult to learn. Based on this observation, Dupont et al.  \cite{dupont2019ANODE} shows that augmenting the feature space by at least one dimension leads to more expressive neural networks with simpler flows. Moreover, they also show that space augmentation improves stability properties. Thus, this method is commonly used and will here be combined with RK Nets.
If the input space is augmented by additional dimensions as described above, the ordinary differential equation defined on this larger space is called \textit{augmented neural ordinary differential equation (ANODE)}, cf. \cite{dupont2019ANODE}. 

\if{We  mention other schemes, as   \cite[sec.\,4]{haber2017stability} who apply the leapfrog and the Verlet integration techniques for optimizing two Hamiltonian system inspired networks, whereas \cite[sec.\,3]{chen2018NODE} make use of the implicit Adam method. Another promising idea is to employ multigrid methods such that the number of layers which corresponds to the grid mesh size varies across different stages of training, see multi-level learning \cite[sec.\,5.3]{haber2017stability} and \cite[p.\,2]{weinan2017MLviaDynamicalSystems}. 
}\fi
To obtain a discrete scheme such that optimization (i.\,e.\ deriving first-order optimality conditions) and discretization commute, we apply partitioned Runge-Kutta methods with non-vanishing weights and coefficients satisfying certain conditions, cf. \eqref{symplectic_RK} below, Sanz-Serna 
\cite{sanzserna2015symplecticRKandMore} and also  Hager \cite{MR1804658}. Since quadratic invariants are preserved in this case, these methods are called \textit{symplectic}. 
 As a result, the backpropagation can be obtained by first discretizing the ordinary differential equation with the proposed scheme and then deriving the adjoint equation or vice versa.
Interpreting this discretized system as a deep neural network has led to new models, which can be trained with methods fitted to their architecture.
We remark that RK Nets have also been considered in Raissi et al.~\cite{MR3881695}. This publication about physics informed networks has recently gained high attention.

\if{We remark that the property that optimization and discretization commute does not hold  only for Runge-Kutta methods satisfying \eqref{symplectic_RK}, but also for other methods, as B-series methods \cite{}.
}\fi%

The paper is organized as follows. In Section \ref{sec:ResNets-ODE} we recall the relation between ordinary differential equations and neural networks, in Section \ref{sec:training} we consider the training problem, and in Section \ref{sec:numerics} we present several numerical examples for point classification problems implemented in PyTorch and extend these experiments to image data. Finally, Section \ref{sec:outlook} gives an outlook on possible future research.


\section{Relation of residual networks to ordinary differential equations}\label{sec:ResNets-ODE}

\label{sec:ResNet_to_ODE}


We recall the basic structure of a feed-forward network, also known as \textit{multilayer perceptron (MLP)}, and of a residual network (ResNet) which can be associated with ordinary differential equations. 

We consider a network with $L$ layers and $d^{[l]}$ neurons in the $l$-th layer for $l = 0,\dots,L$.
 The first layer is called \textit{input layer} and the last layer is called \textit{output layer}. If there are layers in between (i.\,e.\  $L\geq 2$), these are called \textit{hidden layers} and the  neural network is called \textit{deep}, see 
 \cite{MR4027841}.
While a neuron of the input layer is fed by the given data directly, a neuron in one of the following layers gets its input from neurons in previous layers and transforms it to give a new output. The output of each layer is denoted by $y^{[l]} \in \R^{d^{[l]}}$, where the $i$-th entry in this vector corresponds to the output of the $i$-th neuron in the $l$-th layer.
Formally, the operations within a feed-forward network are defined as follows: We consider the commonly used activation functions $ \sigma(x): \R \to \R$ acting componentwise on the arguments and given by one of the functions
\be
\begin{aligned}
\operatorname{ReLU}(x) &:= \max(0, x);&
\operatorname{softplus}(x) &:= log(1 + e^x );\\
\operatorname{logit}(x)&:=1/(1+ e^{-x});&
\operatorname{tanh}(x) &:= 2 \operatorname{logit}(2x) - 1
\end{aligned}
\ee
for $x \in \R$.
For given input data $x \in \R^d$ with $d \in \NN$ we set $y^{[0]}:=x$ (i.\,e.\ $d^{[0]}=d$) and compute for the layers  $l=0,\ldots,L-1$
\begin{align}\label{iter-z}
    z^{[l+1]} &:=K^{[l]} y^{[l]} + b^{[l]}\in \R^{d^{[l+1]}} \quad \text{and} \quad
    y^{[l+1]}:=\sigma\left(z^{[l+1]}\right)
\end{align}
with weight matrix $K^{[l]} \in \R^{d^{[l+1]} \times d^{[l]}}$ and bias vector $b^{[l]} \in \R^{d^{[l+1]}}$. \if{The $i$-th neuron in the $(l+1)$-th layer for $i \in \{1,\ldots, d^{[l+1]}\}$ produces the value
\begin{align}
    y^{[l+1]}_i = \sigma\left(\sum_{j=1}^{d^{[l]}} K^{[l]}_{ij} y^{[l]}_j+b^{[l]}_i\right)
\end{align}
from the output values of all neurons in the $l$-th layer. }\fi

Generally, the number of neurons per layer can vary in neural networks. The dimensions of layers linked via a residual connection, however, have to be equal. This motivates the approach of keeping the width of the network constant across all layers. Consequently, the input data dimension $d$ would determine the number of neurons in each hidden layer, as well as in the output layer. To ease this restriction on the network architecture, we can add $d^\ast \in \NN$ dimensions to the input data $x \in \R^d$ by filling the additional components by zeros, resulting in the augmented input 
\be 
\hat{x}=(x^\top,0,\ldots,0)^\top \in \R^{d+d^\ast}.
\ee 
This method is known as \textit{feature space augmentation}  and results in significant improvements with respect to expressivity, training, and generalization of the network model, see Dupont et al. \cite{dupont2019ANODE}. We obtain
$d^{[l]}=\hat{d}$ for all $l=0,\ldots,L$ where the dimension of the feature space is determined by $\hat{d}:=d+d^\ast$. That way, we have
\begin{align}
    u^{[l]}=\left(K^{[l]},b^{[l]}\right) \in \R^{\hat{d}\times\hat{d}} \times \R^{\hat{d}} \quad \text{for all } l=0,\ldots,L-1
\end{align}
resulting in the  constant number $m \coloneqq \hat{d}^2 + \hat{d}$ of parameters per layer. Hence, we can define a function $f:\R^m \times \R^{\hat{d}} \to \R^{\hat{d}}$ by 
\be
f(u,y):=\sigma(K y + b),\quad u=(K,b)\in \R^{\hat{d}\times\hat{d}} \times \R^{\hat{d}},
\ee
%
so that the transformation between layers as specified in \eqref{iter-z} becomes 
\begin{align}
    y^{[l+1]}=f(u^{[l]},y^{[l]}) \quad \text{for all } l=0,\ldots,L-1
    \label{feedforward_net}
\end{align}
with $y^{[0]}=\hat{x}$.

Alternatively, we can introduce the forward propagation for $h>0$ by 
\begin{align}
    y^{[l+1]}=y^{[l]}+hf(u^{[l]},y^{[l]}) \quad \text{for all } l=0,\ldots,L-1.
    \label{Euler}
\end{align}
For $h=1$ we recover a conventional ResNet. 
Formally, these networks can be interpreted as explicit forward Euler discretizations of ordinary differential equations  of type
\begin{align}
    \dot{\y}(t)&=\sigma(\mathbf{K}(t) \y(t) + \bfunc(t)) \quad \text{for a.\,a. } t \in [0,T],\quad \y(0)=\hat{x}
    \label{ODE_NN}
\end{align}
with appropriate chosen $\mathbf{K}\colon [0,T] \rar \R^{\dh \times \dh}$, $\mathbf{b}\colon [0,T]\rar \R^{\dh}$ and $\y\colon [0,T] \rar \R^{\dh}$. Setting $\ufunc=(\mathbf{K},\bfunc)$, the introduced functions are required to match the model parameters and outputs of the respective layers at time $t_l=lh$ with $h=T/L$, i.\,e. 
\begin{equation}
\begin{aligned}
    \ufunc(t_l) & = u^{[l]} \quad \text{for }l=0,\ldots, L-1,\\
    \y(t_l) & = y^{[l]} \quad \text{for } l=0,\ldots, L.
    \label{t_disc}
\end{aligned}
\end{equation}
Furthermore, we can reformulate \eqref{ODE_NN} as 
\begin{align}\label{ODE-cont}
    \dot{\y}(t)&=f(\ufunc(t),\y(t)) \quad \text{for a.\,a. } t \in [0,T],\quad \y(0)=\hat{x}.
\end{align}

\if{In order to specify the function $f$ describing the dynamics of the ODE \eqref{ODE_NN}, we note that the model parameters of each layer consist of a weight matrix and a bias vector, that is $u^{[l]}=(K^{[l]},b^{[l]})$. Via \eqref{t_disc}, the same structure is imposed on the function $\ufunc$, leading to a time-dependent matrix and vector, denoted by $\ufunc(t)=(\K(t),\bfunc(t))$. Furthermore, we recall that $f$ is a non-linear activation after an affine transformation, see \eqref{f_disc}. Hence, the right-hand side of \eqref{ODE_NN} is given by
\begin{align}
    f(\ufunc(t),\y(t))= \sigma(\K(t) \y(t) + \bfunc(t)), \quad t \in [0,T],
    \label{f_cont}
\end{align}
cf.\ \cite[p.\,5]{haber2017stability}.}\fi


The focus of this paper is placed on the numerical performance of the considered approach and not on an analytical consideration. Therefore, we do not discuss existence and uniqueness for the differential equation here. In the following, we assume there exists a well-posed solution operator mapping $\ufunc$ to the solution $\y[\ufunc]$.

Starting now from \eqref{ODE-cont} we derive Runge-Kutta Nets
by discretizing the equation with RK schemes determined by triples $(A,\beta, c)\in \R^{s\times s} \times \R^s \times \R^s$ according to the RK tableau shown in Figure~\ref{fig:Butcher_tableau}, where  $s$ denotes the number of stages, $A=(a_{i,j})_{i,j=1}^s$ the RK matrix, $\beta = (\beta_i)_{i=1}^s$ the weights, and $c = (c_i)_{i=1}^s$ the nodes.
\begin{figure}
    \centering
    \begin{minipage}{.24\textwidth}
    \begin{align*}
    \begin{array}
    {c|c}
    c & A\\
    \hline
    & \beta^\top 
    \end{array}
    \end{align*}
    \end{minipage}
    \begin{minipage}{.24\textwidth}
    \begin{align*}
    \begin{array}
    {c|c}
    0\\
    \hline
    & 1 
    \end{array}
    \end{align*}
    \end{minipage}
    \begin{minipage}{.34\textwidth}
    \begin{align*}
    \begin{array}
    {c|cccc}
    0\\
    \frac{1}{2} & \frac{1}{2}\\
    \frac{1}{2} &0 &\frac{1}{2} \\
    1& 0& 0& 1\\
    \hline
    & \frac{1}{6} &\frac{1}{3} &\frac{1}{3} &\frac{1}{6} 
    \end{array}
    \end{align*}
    \end{minipage}
    \caption[Butcher tableaus]{Butcher tableaus: (from left to right) general form, forward Euler and classic RK4.
    }
    \label{fig:Butcher_tableau}
\end{figure}
Then,  for  $l=0,\dots, L-1$ the scheme is given by
\begin{equation}
\begin{aligned}
    & y^{[l+1]}= y^{[l]} + h \sum_{i=1}^s \beta_i f(u^{[l]}_i,y^{[l]}_i) \quad \text{with}\\ 
    &y^{[l]}_i = y^{[l]} +  h \sum_{j=1}^s a_{i,j} f(u^{[l]}_j,y^{[l]}_j) \quad \text{and} \quad u^{[l]}_i = \ufunc(t_l + h c_i)
    \quad \text{for all } i=1,\dots, s,
    \label{state_equ_RK}
\end{aligned}
\end{equation}
see, e.\,g., \cite{sanzserna1992symplecticRK,sanzserna2015symplecticRKandMore}. Note that the method is explicit if $a_{i,j}=0$ for all $j \geq i$ and otherwise implicit. Apart from the forward Euler method, the classic RK or so-called RK4 method is one of the most frequently used explicit methods. 
%
\if{
\begin{figure}
    \centering
    \begin{minipage}{.24\textwidth}
    \begin{align*}
    \begin{array}
    {c|c}
    c & A\\
    \hline
    & \beta^\top 
    \end{array}
    \end{align*}
    \end{minipage}
    \begin{minipage}{.24\textwidth}
    \begin{align*}
    \begin{array}
    {c|c}
    0\\
    \hline
    & 1 
    \end{array}
    \end{align*}
    \end{minipage}
    \begin{minipage}{.24\textwidth}
    \begin{align*}
    \begin{array}
    {c|c}
    1 & 1\\
    \hline
    & 1 
    \end{array}
    \end{align*}
    \end{minipage}
    \begin{minipage}{.24\textwidth}
    \begin{align*}
    \begin{array}
    {c|cccc}
    0\\
    \frac{1}{2} & \frac{1}{2}\\
    \frac{1}{2} &0 &\frac{1}{2} \\
    1& 0& 0& 1\\
    \hline
    & \frac{1}{6} &\frac{1}{3} &\frac{1}{3} &\frac{1}{6} 
    \end{array}
    \end{align*}
    \end{minipage}
    \caption[Butcher tableaus]{Butcher tableaus: (from left to right) general form, forward Euler, backward Euler and classic RK4. Adapted from \cite[p.\,294]{sanzserna1992symplecticRK} and \cite[p.\,21]{benning2019DLasOCP}.}
    \label{fig:Butcher_tableau}
\end{figure}
}\fi
%
\if{
When examining the iteration in \eqref{state_equ_RK}, we see that for each step not only the value of the control function $u^{[l]}$ at time $t_l=lh$ is required, but a whole set of discrete function values $(u^{[l]}_i)_{i=1}^s$ needs to be extracted from the control via the RK nodes $(c_i)_{i=1}^s$. }\fi
In practice, the internal stages $(u^{[l]}_i)_{i=1}^s$ are replaced by the single value $u^{[l]}$, see \cite{benning2019DLasOCP}. The simplified numerical scheme is thus given by
\begin{equation}
\begin{aligned}
    &y^{[l+1]}  = y^{[l]} + h \sum_{i=1}^s \beta_i f(u^{[l]},y^{[l]}_i)\\
    \text{with} \quad 
    &y^{[l]}_i = y^{[l]} +  h \sum_{j=1}^s a_{i,j} f(u^{[l]},y^{[l]}_j) \quad \text{for all } i=1,\dots, s.
    \label{simplified_RK}
\end{aligned}
\end{equation}
With that simplification, the RK coefficient $c$ does not appear in the discretization anymore and is therefore irrelevant, as is the case for all autonomous systems \cite{sanzserna2015symplecticRKandMore}. 
In the scope of this paper, only RK methods are considered, although this approach can be extended to a wide range of symplectic integration techniques.
\section{The training problem for classification}\label{sec:training}

Let us consider a classification problem with $K$ classes. For assigning a data sample $x \in \R^d$ to a particular class, we choose a weight matrix $W \in \R^{K \times \hat{d}}$ and bias vector $\mu \in \R^K$ to construct an affine classifier, as well as the $\operatorname{softmax}$ function $\calh \colon \R^K \rar \R^K$ defined by $ \calh_k(z) :=e^{z_k}/(\sum^K_{j=1}e^{z_j})$ to transform the result into a probability vector. Let $y \in \R^{\hat{d}}$ be the output when feeding a given model with the input $x$. Then, we classify according to the following rule: $x$ belongs to the class associated to the largest entry in the vector $\mathcal{H}\left(W y+\mu\right)$.

For supervised learning, let the training data be given as $((x_n,c_n))_{n=1}^N$ consisting of the samples $x_n \in \R^d$ and the labels $c_n \in \R^K$ with entries representing the probability of the sample belonging to the respective classes. In order to train a deep learning model, we require a suitable cost function. Considering the model defined by the continuous dynamical system \eqref{ODE-cont}, we define $\y_n[\ufunc]$ to be the solution of the ordinary differential equation for a given parameter $\ufunc$ and with initial value $x_n$. Using the cross-entropy loss, the cost $\hat{\calf}\colon L^\infty(0,T;\R^m) \rar \R$ is given by 
\begin{align}
    \hat{\calf}\left(\ufunc\right):=\frac{1}{N} \sum_{n=1}^N -\mathds{1}^\top \left[c_n \circ \log\left(\mathcal{H}\left(W \y_n[\ufunc](T)+\mu\right)\right)\right].
    \label{F_cont_CE}
\end{align}
where $\mathds{1}$ is the vector of all ones and $\circ$ denotes the componentwise multiplication of vectors, also known as the Hadamard product. Often, an appropriate regularization term is included in order to avoid overfitting from a numerical viewpoint or to guarantee certain theoretical properties from an analytical viewpoint. Since the latter is not addressed in this paper and the former aspect does not appear in the numerical examples presented in Section \ref{sec:numerics}, we do not consider regularizers here. Thus, the training problem reads
\begin{align}\label{problem-cont}
    \min \hat{\calf}\left(\ufunc\right),\; \text{s.t. } \eqref{ODE-cont}.
\end{align}
We remark that it is common practice to optimize not only with respect to the model parameters~$\ufunc$, but also with respect to the parameters of the affine classifier $W$ and $\mu$. 

Interpreting $\ufunc$ as control and $\y$ as state, \eqref{problem-cont} can be viewed as an optimal control problem with ODE constraint. Generally, there are two approaches to solve this problem numerically: We can either discretize the optimal control problem first and then derive first-order  optimality conditions for this discrete problem formulation, or we can take the optimality conditions of the continuous optimal control problem and discretize them. We choose a discretization such that the optimization and discretization formally commute.

Following the first-discretize-then-optimize approach also known as direct approach \cite[Sec.\,4.3]{sanzserna2015symplecticRKandMore}, we will derive the first-order conditions for optimality for the discrete optimal control problem formally. If Runge-Kutta Nets are applied, this becomes the network training problem with discrete cost $\calf\colon \R^{m \cdot L} \rar \R$, that is
\begin{align}\label{problem-discr}
    \min \calf(u):=\frac{1}{N} \sum_{n=1}^N -\mathds{1}^\top \left[c_n \circ \log\left(\mathcal{H}\left(W y_n[u]^{[L]}+\mu\right)\right)\right],\; \text{s.t. } \eqref{state_equ_RK}.
\end{align}
Here, we define $y_n[u]^{[l]}$ to be the output of the $l$-th layer when applying the network with given parameters $u=\left(u^{[l]}\right)_{l=0}^{L-1}$ to the training sample $x_n$. 
\if{\be
\calr(u):=\frac{\lambda}{2}\norm{\mathbf{K}}{\infty}
\ee
with 
\be
\norm{\mathbf{K}}{\infty}:=XXX.
\ee}\fi
\if{
Setting $U_{ad}=\R^m$ with the pointwise formulation in \eqref{U_ad}, the training problem is given by 
Discretizing the control and the state in time as in \eqref{t_disc} and choosing a suitable numerical method for solving the IVP denoted by $\D_f^h: \R^m \times \R^{\hat{d}} \to \R^{\hat{d}}$, we arrive at the discretized OCP 
\be
\begin{aligned}
    &\min_{u, y} F(y^{[L]}) \text{ s.t. }\\ 
    & y^{[l+1]}=\D_f^h(u^{[l]},y^{[l]}) \text{ for all } l=0,\dots,L-1 \text{ and } y^{[0]}=\hat{x},\\
    & u=(u^{[l]})_{l=0}^{L-1},\\ &y=(y^{[l]})_{l=0}^{L},\text{ where } u^{[l]} \in \R^m\text{ and }y^{[l]} \in \R^{\hat{d}}\text{ for all }l.
    \label{OCP_disc}
\end{aligned}
\ee
 }\fi
We recall from \cite{benning2019DLasOCP} the derivation of the discrete optimality conditions of
\eqref{problem-discr} where the equation is discretized by a RK method defined iteratively for $l=0,\dots,L-1$ by
\begin{align}
    y^{[l+1]} &= y^{[l]} + h \sum_{i=1}^s \beta_i f^{[l]}_i, \label{state_equ_RK_1}\\
    f^{[l]}_i &= f(u^{[l]}_i,y^{[l]}_i) \quad \text{for all } i=1,\dots, s \quad \text{and} \label{state_equ_RK_2}\\
    y^{[l]}_i &= y^{[l]} +  h \sum_{j=1}^s a_{i,j} f^{[l]}_j \quad \text{for all } i=1,\dots, s \label{state_equ_RK_3}
\end{align}
with initial condition $y^{[0]}=\hat{x}$, see, e.\,g., 
\cite{benning2019DLasOCP, sanzserna2015symplecticRKandMore}.
Note that this system can be reduced to two equations by removing one of the redundant variables: For instance, we could eliminate $f_i^{[l]}$ by plugging \eqref{state_equ_RK_2} into \eqref{state_equ_RK_1} and \eqref{state_equ_RK_3}. This yields the previous formulation \eqref{state_equ_RK} of a general RK scheme. However, for defining the Lagrangian, we will instead remove $y_i^{[l]}$ by plugging \eqref{state_equ_RK_3} into \eqref{state_equ_RK_2} because this form will be more tractable in further computations. Hence,
for 
\be
\begin{aligned}
    U^{[l]}&:=(u_1^{[l]},\dots,u^{[l]}_s)^{\top};& U&:=(U^{[0]},\dots,U^{[L-1]})^{\top}\in \R^{m \times s \times L};\\
    F^{[l]}&:=(f_1^{[l]},\dots,f^{[l]}_s)^{\top};& F&:=(F^{[0]},\dots,F^{[L-1]})^{\top}\in \R^{\hat{d} \times s \times L};\\
      \Xi^{[l]}&:=(\xi_1^{[l]},\dots,\xi^{[l]}_s)^{\top};& \Xi&:=(\Xi^{[0]},\dots,\Xi^{[L-1]})^{\top}\in \R^{\hat{d} \times s \times L};\\
      Y&:=(y^{[0]},\dots,y^{[L]})^{\top} \in \R^{\hat{d} \times (L+1)};& P&:=(p^{[0]},\dots,p^{[L]})^{\top}\in \R^{\hat{d} \times (L+1)};\\
\end{aligned}
\ee
we define the Lagrangian $\Lagr= \Lagr \left(U,Y,F,P,\Xi\right)$ with
\be
\begin{aligned}
    \Lagr &:= \calf\left(y^{[L]}\right) + p^{[0]} \cdot \left(\hat{x}-y^{[0]}\right)
    + \sum_{l=0}^{L-1}\left[ p^{[l+1]} \cdot \left(y^{[l]} + h \sum_{i=1}^s \beta_i f_i^{[l]}-y^{[l+1]}\right) \right.\\
    &\quad \left. + \sum_{i=1}^s \xi_i^{[l]} \cdot \left(f\Big(u_i^{[l]},y^{[l]} +  h \sum_{j=1}^s a_{i,j} f^{[l]}_j\Big)-f_i^{[l]}\right)\right]
\end{aligned}
\ee
with Lagrange multipliers $P$ and $\Xi$. 

\if{If we differentiate $\Lagr$ with respect to all Lagrange multipliers $p^{[l]}$ and $\xi_i^{[l]}$ and set the derivatives to zero, we obtain as expected the RK discretization of the state equation as formulated in \eqref{state_equ_RK_1}--\eqref{state_equ_RK_3} with its initial condition. }\fi

Setting the derivative of $\Lagr$ with respect to all discrete state values $y^{[l]}$ in direction $z^{[l]}$ to zero yields
\be
\begin{aligned}
    0 &= \nabla \calf\left(y^{[L]}\right) \cdot z^{[L]} + p^{[0]} \cdot \left(-z^{[0]}\right)\\
    &\quad + \sum_{l=0}^{L-1}\left[ p^{[l+1]} \cdot \left(z^{[l]} -     z^{[l+1]}\right) + \sum_{i=1}^s \xi_i^{[l]} \cdot f_y \Big(u_i^{[l]},y^{[l]} +  h \sum_{j=1}^s a_{i,j} f^{[l]}_j\Big)z^{[l]}\right].
\end{aligned}
\ee
Reordering the summands, we get
\be
\begin{aligned}
    0 &= \left(\nabla \calf\left(y^{[L]}\right) - p^{[L]}\right) \cdot z^{[L]}\\
    &\quad + \sum_{l=0}^{L-1}\left[\left(p^{[l+1]}-p^{[l]} + \sum_{i=1}^s f_y \Big(u_i^{[l]},y^{[l]} +  h \sum_{j=1}^s a_{i,j} f^{[l]}_j\Big)^\top \xi_i^{[l]}\right) \cdot z^{[l]}\right].
\end{aligned}
\ee
Since the directions $z^{[l]} \in \R^{\hat{d}}$ are arbitrary, it follows
\begin{align}
    0 &= \nabla \calf\left(y^{[L]}\right) - p^{[L]} \quad \text{and}\\
    0 &= p^{[l+1]}-p^{[l]} + \sum_{i=1}^s f_y \Big(u_i^{[l]},y^{[l]} +  h \sum_{j=1}^s a_{i,j} f^{[l]}_j\Big)^\top \xi_i^{[l]} \quad \text{for all } l=0,\dots,L-1.
\end{align}
This is equivalent to the iteration
\begin{align}
   p^{[l+1]} = p^{[l]} - \sum_{i=1}^s f_y \Big(u_i^{[l]},y^{[l]} +  h \sum_{j=1}^s a_{i,j} f^{[l]}_j\Big)^\top \xi_i^{[l]} \quad \text{for all } l=0,\dots,L-1
   \label{Lagr_deriv_1}
\end{align}
with boundary condition $p^{[L]} = \nabla \calf\left(y^{[L]}\right)$.

Next, we let the derivative of $\Lagr$ with respect to $f_i^{[l]}$ in direction $\delta f_i^{[l]}$ vanish. That is
\be
\begin{aligned}
  0 &= \sum_{l=0}^{L-1}\left[ p^{[l+1]} \cdot h \sum_{i=1}^s \beta_i \delta f_i^{[l]} \right. \\
  &\quad + \left. \sum_{i=1}^s \xi_i^{[l]} \cdot \left(f_y \Big(u_i^{[l]},y^{[l]} +  h \sum_{j=1}^s a_{i,j} f^{[l]}_j\Big) h \sum_{j=1}^s a_{i,j} g^{[l]}_j-\delta f_i^{[l]}\right)\right].
\end{aligned}
\ee
Changing the order of summation, we obtain
\begin{align}
  0 = \sum_{l=0}^{L-1}\sum_{i=1}^s \left(h \beta_i p^{[l+1]} - \xi_i^{[l]} + h \sum_{j=1}^s a_{j,i} f_y \Big(u_j^{[l]},y^{[l]} +  h \sum_{k=1}^s a_{j,k} f^{[l]}_k\Big)^\top \xi_j^{[l]}\right) \cdot \delta f_i^{[l]}.
\end{align}
Because each of the directions $\delta f_i^{[l]} \in \R^{\hat{d}}$ is arbitrary, we have for all $l=0,\dots,L-1$ and $i=0,\dots,s$
\begin{align}
  0 = h \beta_i p^{[l+1]} - \xi_i^{[l]} + h \sum_{j=1}^s a_{j,i} f_y \Big(u_j^{[l]},y^{[l]} +  h \sum_{k=1}^s a_{j,k} f^{[l]}_k\Big)^\top \xi_j^{[l]},
\end{align}
or equivalently
\begin{align}
   \xi_i^{[l]} = h \left[\beta_i p^{[l+1]} + \sum_{j=1}^s a_{j,i} f_y \Big(u_j^{[l]},y^{[l]} +  h \sum_{k=1}^s a_{j,k} f^{[l]}_k\Big)^\top \xi_j^{[l]}\right].
   \label{Lagr_deriv_2}
\end{align}

Finally, we consider the derivative  of $\Lagr$ with respect to the discrete control values $u_i^{[l]}$ in direction $v_i^{[l]}$ and set them to zero. This yields
\be
\begin{aligned}
    0 &= \sum_{l=0}^{L-1} \sum_{i=1}^s \xi_i^{[l]} \cdot f_u \Big(u_i^{[l]},y^{[l]} +  h \sum_{j=1}^s a_{i,j} f^{[l]}_j\Big)v_i^{[l]}\\
     &= \sum_{l=0}^{L-1} \sum_{i=1}^s f_u \Big(u_i^{[l]},y^{[l]} + h \sum_{j=1}^s a_{i,j} f^{[l]}_j\Big)^\top\xi_i^{[l]} \cdot v_i^{[l]}.
\end{aligned}
\ee
Using again the fact that the directions $v_i^{[l]} \in \R^m$ can be chosen arbitrarily, we get for all $l=0,\dots,L-1$ and $i=1,\dots,s$
\begin{align}
    0 = f_u \Big(u_i^{[l]},y^{[l]} +  h \sum_{j=1}^s a_{i,j} f^{[l]}_j\Big)^\top\xi_i^{[l]}.
    \label{Lagr_deriv_3}
\end{align}

Now, we assume that $\beta_i \neq 0$ for all $i=1,\dots,s$. Then, we rewrite the results \eqref{Lagr_deriv_1}, \eqref{Lagr_deriv_2} and \eqref{Lagr_deriv_3} with $p_i^{[l]} \coloneqq \xi_i^{[l]}/(h\beta_i)$ and $y^{[l]}_i$ as defined in \eqref{state_equ_RK_3} for all $l=0,\dots,L-1$ and $i=1,\dots,s$. This leads to the iteration
\begin{align}
    p^{[l+1]} = p^{[l]} - h \sum_{i=1}^s \beta_i f_y \left(u_i^{[l]},y_i^{[l]}\right)^\top p_i^{[l]} \quad \text{for } l=0,\dots,L-1, \quad p^{[L]} = \nabla \calf\left(y^{[L]}\right)
\end{align}
with
\be
\begin{aligned}
    p_i^{[l]} &= p^{[l+1]} + h \sum_{j=1}^s \frac{a_{j,i}\beta_j}{\beta_i} f_y \left(u_j^{[l]},y_j^{[l]}\right)^\top p_j^{[l]}\\
    &= p^{[l]} - h \sum_{j=1}^s \left(\beta_j - \frac{a_{j,i}\beta_j}{\beta_i}\right) f_y \left(u_j^{[l]},y_j^{[l]}\right)^\top p_j^{[l]} \quad \text{for all } i=1,\dots,s,
\end{aligned}
\ee
as well as to the first order necessary condition for optimality
\begin{align}
    f_u \left(u_i^{[l]},y_i^{[l]}\right)^\top p_i^{[l]} = 0 \quad \text{for all } l=0,\dots,L-1 \text{ and } i=1,\dots,s.
    \label{optf_cond_RK}
\end{align}
Defining coefficients $(\tilde{A}, \tilde{\beta}, \tilde{c})$ via the conditions
\begin{align}
    \beta_i = \tilde{\beta_i}, \quad \beta_i\tilde{a_{i,j}}+\tilde{\beta_j} a_{j,i}-\beta_i \tilde{\beta_j}=0, \quad c_i=\tilde{c}_i \quad \text{for all } i,j=1,\dots,s
    \label{symplectic_RK}
\end{align}
\if{as, for instance, in \cite[p.\,9]{benning2019DLasOCP}, \cite[p.\,40]{bonnans2008OCofODEs}, \cite[p.\,295]{sanzserna1992symplecticRK} or \cite[p.\,7]{sanzserna2015symplecticRKandMore},}\fi we obtain
\begin{align}
    p^{[l+1]} &= p^{[l]} + h \sum_{i=1}^s \tilde{\beta}_i g^{[l]}_i, \label{adj_equ_RK_1}\\
    g^{[l]}_i &= -f_y(u^{[l]}_i,y^{[l]}_i)^\top p^{[l]}_i \quad \text{for all } i=1,\dots, s \quad \text{and} \label{adj_equ_RK_2}\\
    p^{[l]}_i &= p^{[l]} +  h \sum_{j=1}^s \tilde{a}_{i,j} g^{[l]}_j \quad \text{for all } i=1,\dots, s \label{adj_equ_RK_3}
\end{align}
with final condition $p^{[L]} = \nabla F(y^{[L]})$. \if{, see, e.\,g., \cite[p.\,8]{benning2019DLasOCP}, \cite[p.\,40]{bonnans2008OCofODEs} and \cite[p.\,19]{sanzserna2015symplecticRKandMore}. }\fi 

When deriving the adjoint equation of \eqref{problem-cont} in the first-optimize-then-discretize or so-called indirect approach, we find that \eqref{adj_equ_RK_1}--\eqref{adj_equ_RK_3} is precisely its RK discretization. Together with the RK discretization of the state equation given in \eqref{state_equ_RK_1}--\eqref{state_equ_RK_3}, this system is called \textit{partitioned RK method}.
Furthermore, it can be shown that \eqref{optf_cond_RK} are the discretized optimality conditions of the continuous optimal control problem. Hence, under the conditions \eqref{symplectic_RK} deriving optimality conditions and discretizing commute (i.e. formally the direct and indirect approach are equivalent), cf., e.g., 
\cite{benning2019DLasOCP, sanzserna1992symplecticRK}.
\if{We note that this is precisely a RK discretization of the adjoint equation \eqref{adj_equ_NN} and  \eqref{opt_cond_RK} are the discretized optimality conditions of the continuous control problem as derived in \eqref{opt_cond_NN}. Hence, the method of constructing optimality conditions and then discretizing them, which is known as the indirect approach 
is formally equivalent to the direct approach presented here. 
}\fi

\if{
\section{Experimental Design}

In order to test ODE based neural networks, we will perform numerical experiments. As in the previous chapter, we will focus on RK methods for discretizing the underlying ODE, and thus call the resulting models \textit{RK Nets}. These networks are employed to solve simple point classification tasks, in which data points in a finite dimensional space and a finite number of possible classes are given. The goal is to assign a class to each point according to its position in space. Since we pose a supervised learning problem, the NN is trained on a set of labelled points before it is tested on another set of points, whose labels are only used to evaluate the performance of the model.

The experiments are conducted on various datasets with points in different dimensions belonging to two or more classes. We will compare the performance of several network architectures on these sets while modifying their hyperparameters. These include not only the parameters referring to the model structure such as activation function, depth and width, but also the ones determining the training process. Depending on the optimization algorithm, these are for instance batch size, number of epochs, learning rate and momentum, see, e.\,g., \cite{radhakrishnan2017hyperparameters}. 

Furthermore, we will examine the inner workings of the networks by observing the transformation of the data points in the feature space while moving through the layers and tracing their trajectories. Interpreting the flow of data through the network, we intend to give explanations of why a certain network has a poor or high prediction accuracy on a given dataset.
}\fi


\section{Numerical results and interpretation}\label{sec:numerics}

In this section, we will conduct numerical experiments on point classification and interpret their results. The network architectures which will be tested here are the two Runge-Kutta Nets, \verb|EulerNet| and \verb|RK4Net|, defined by the forward propagation \eqref{simplified_RK} with RK coefficients given by their respective tableaus in Figure \ref{fig:Butcher_tableau}. Their performance will be compared to a conventional feed-forward network implemented as \verb|StandardNet| given in  \eqref{feedforward_net}. All point datasets used in the experiments are created manually. For binary classification in two dimensions, the sets are inspired by \cite{benning2019DLasOCP} and depicted in Figure \ref{fig:data_examples}. Moreover, the networks are applied to higher dimensional datasets with multiple classes as illustrated in Figure \ref{fig:comparison_datasets}. To shorten notation, the dimension is abbreviated as D and the number of classes as C. Before comparing the different neural network architectures, we will examine the effect of some central hyperparameters such as width, depth and activation function in order to tune them in an appropriate way. Then, we will investigate the implications of the obtained results on image classification. Since all implemented loss functions and optimization methods produce very similar results, we will limit our experiments exclusively to the use of cross-entropy loss and the Adam algorithm with a batch size of five. Although we do not observe significant differences when additionally optimizing over $W$ and $\mu$, we include these parameters in the training problem for all numerical examples.

The code for these experiments provided at \cite{ARKN} is written in Python (version~3.6.12) using the deep learning framework PyTorch (version~1.7.1). More precisely, the networks are constructed with the help of pre-implemented layers, which we connected in such a way as to give the desired network architecture. Furthermore, they are trained with loss functions and optimizers provided by PyTorch. Besides that, we have created our own tools for plotting the data, tracking metrics during the training and displaying the resultant prediction. In order to shed light on the data processing within the neural network, it is vital to visualize the transformation of the data points in the feature space while moving through the layers. Since the dimensionality of the feature space is often large, using dimensionality reduction techniques is indispensable for producing meaningful plots. Apart from just selecting some coordinates and ignoring the remaining ones, the code enables us to project embeddings to a lower dimensional space via principal component analysis~(PCA).

\begin{figure}
    \centering
    \includegraphics[trim={0 0 0 1.255cm}, clip, width=3cm]{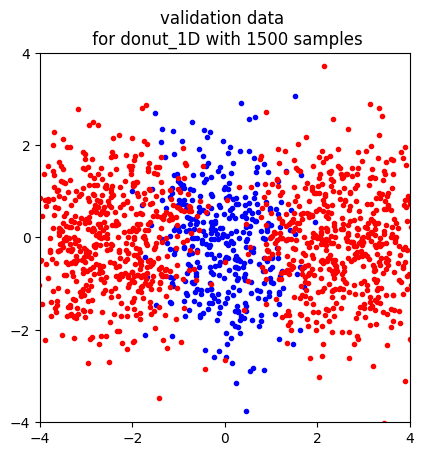}
    \includegraphics[trim={0 0 0 1.255cm}, clip, width=3cm]{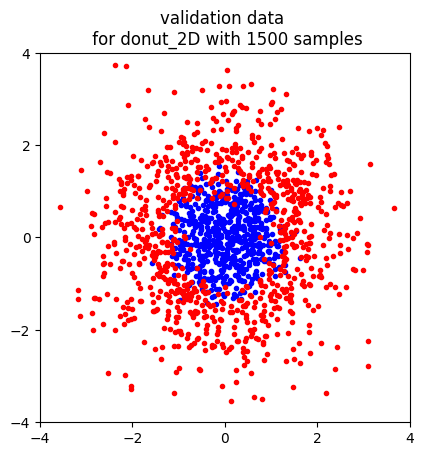}
    \includegraphics[trim={0 0 0 1.255cm}, clip, width=3cm]{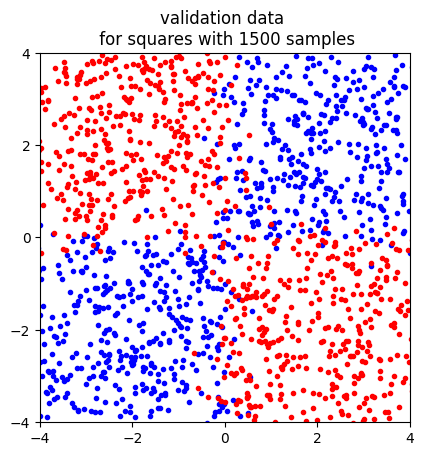}
    \includegraphics[trim={0 0 0 1.255cm}, clip, width=3cm]{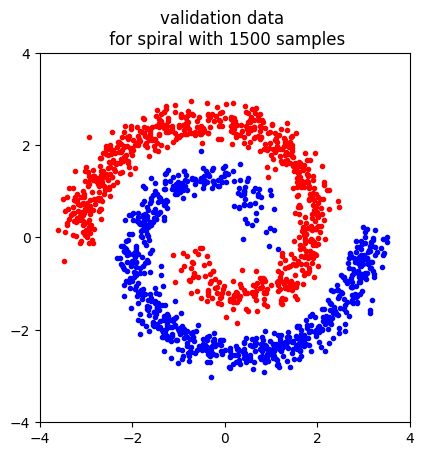}
    \caption[Two dimensional datasets for binary point classification]{Two dimensional datasets for binary point classification with 1500 samples each: (from left to right) \texttt{donut\_1D}, \texttt{donut\_2D}, \texttt{squares} and \texttt{spiral}.}
    \label{fig:data_examples}
\end{figure}

\subsection{Experiments on network width}

\label{sec:width}

First, we aim to determine a suitable width for network models based on ODEs in order to ensure that the choice of this hyperparameter does not hinder their predictive capability. For that, we consider the basic datasets in 2D with only two classes depicted in Figure \ref{fig:data_examples} which exhibit different topological properties.
Note that features of NODEs preserve the topology of the input space, see, e.g., \cite[Proposition 3]{dupont2019ANODE}. This implies that classifying data points in a feature space of the same dimension as the input space may cause difficulties as already described in Section~\ref{sec:intro}.

\begin{figure}
    \centering
    \setlength\tabcolsep{0pt}
    \settowidth\rotheadsize{ANODE}
    \begin{tabularx}{\linewidth}{l ccc}
        & {\small trajectories} & {\small transformation} & {\small prediction}\\ 
        \rothead{\centering NODE}       
        & \includegraphics[width=4.1cm,valign=m]{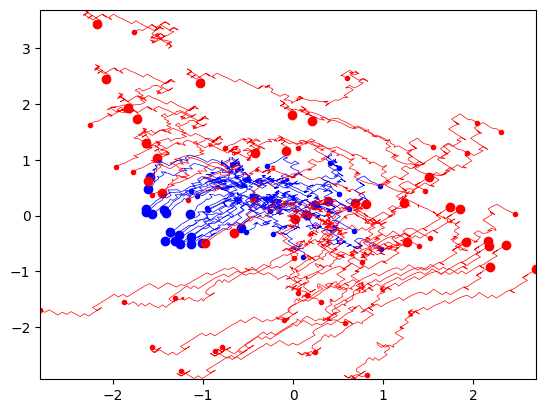}
        & \includegraphics[width=4.1cm,valign=m]{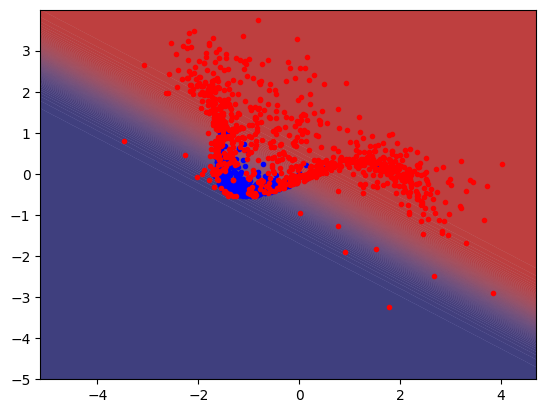}
        &\includegraphics[trim={0 0 0 1.255cm}, clip, width=4.1cm,valign=m]{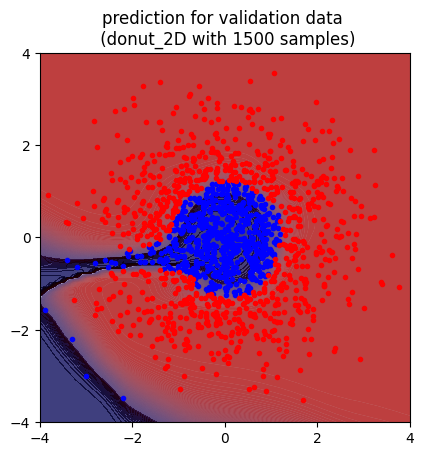}\\ 
        \addlinespace[2pt]
        \rothead{\centering ANODE} 
        & \includegraphics[width=4.1cm,valign=m]{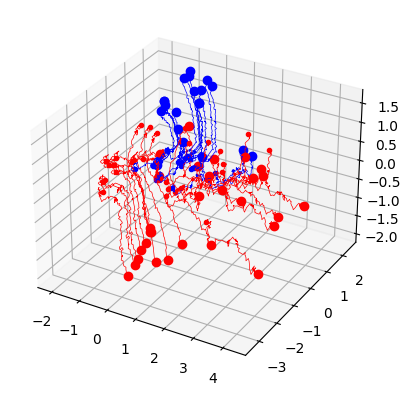}
        & \includegraphics[width=4.1cm,valign=m]{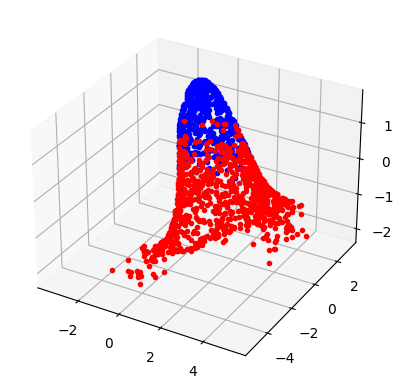}
        &\includegraphics[trim={0 0 0 1.255cm}, clip, width=4.1cm,valign=m]{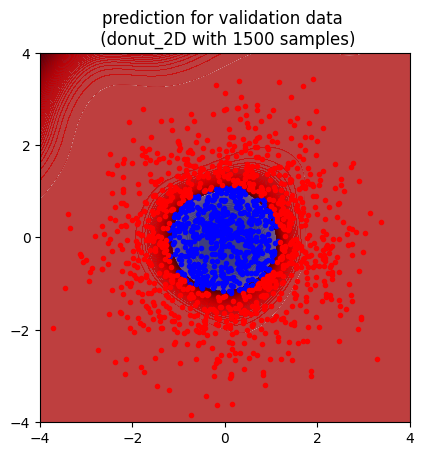}
    \end{tabularx}
    \caption[Classification of \texttt{donut\_2D} with \texttt{RK4Net} based on NODEs or ANODEs]{Classification of \texttt{donut\_2D} with \texttt{RK4Net} of width $\hat{d}=2$ corresponding to the NODE-approach (top) and $\hat{d}=3$, i.\,e.\ with space augmentation characterizing the ANODE-approach (bottom), and of same depth  $L=100$ and $\tanh$ activation. The plots show (from left to right) the trajectories of the features starting at the small dot and terminating at the large dot, their final transformation in the output layer and the resulting prediction with coloured background according to the network's classification.}
    \label{fig:width_donut_2D}
\end{figure}

\begin{figure}
    \centering
    \setlength\tabcolsep{0pt}
    \settowidth\rotheadsize{ANODE}
    \begin{tabularx}{\linewidth}{l ccc}
        & {\small trajectories} & {\small transformation} & {\small prediction}\\ 
        \rothead{\centering NODE}       
        & \includegraphics[width=4.1cm,valign=m]{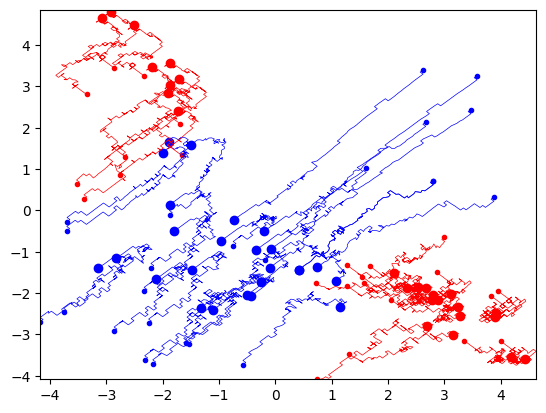}
        & \includegraphics[width=4.1cm,valign=m]{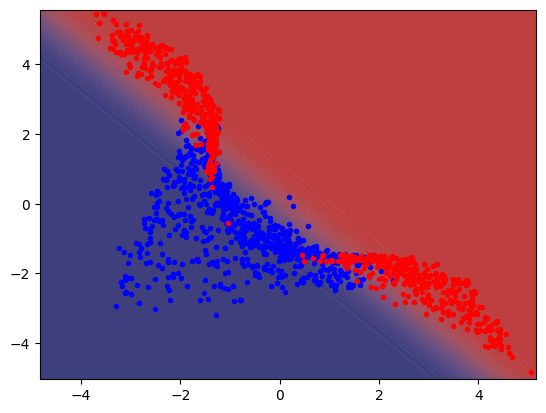}
        &\includegraphics[trim={0 0 0 1.255cm}, clip, width=4.1cm,valign=m]{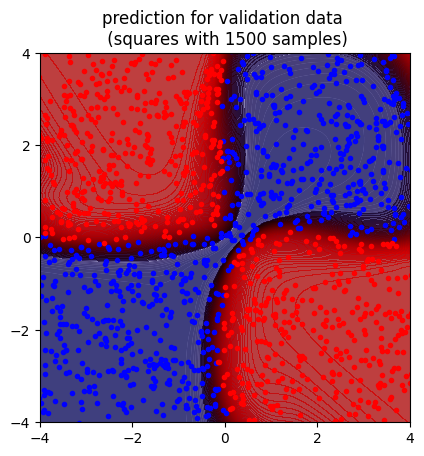}\\ 
        \addlinespace[2pt]
        \rothead{\centering ANODE} 
        & \includegraphics[width=4.1cm,valign=m]{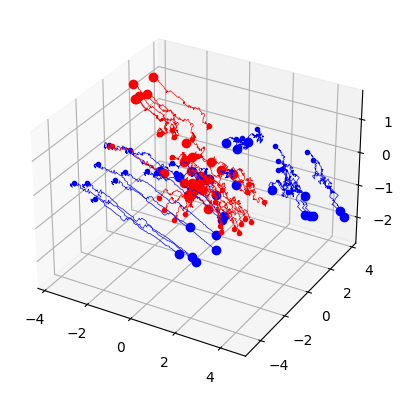}
        & \includegraphics[width=4.1cm,valign=m]{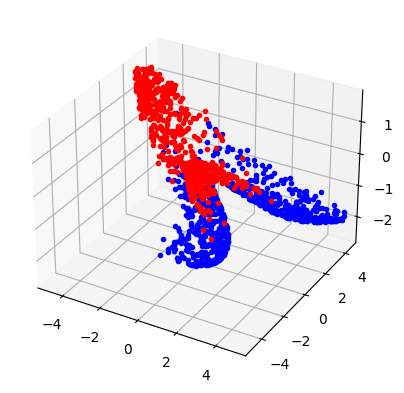}
        &\includegraphics[trim={0 0 0 1.255cm}, clip, width=4.1cm,valign=m]{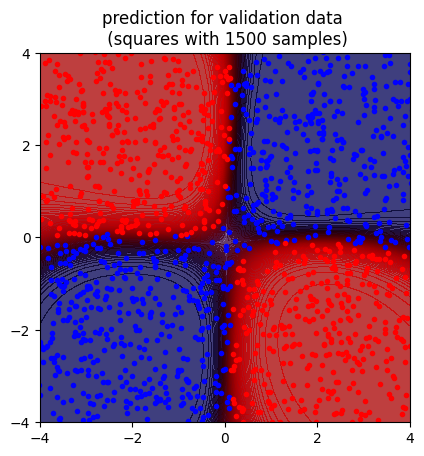}
    \end{tabularx}
    \caption[Classification of \texttt{squares} with \texttt{RK4Net} based on NODEs or ANODEs]{Classification of \texttt{squares} with \texttt{RK4Net} of width $\hat{d}=2$ corresponding to the NODE-approach (top) and $\hat{d}=3$, i.\,e.\ with space augmentation characterizing the ANODE-approach (bottom), and of same depth  $L=100$ and $\tanh$ activation. The plots show (from left to right) the trajectories of the features starting at the small dot and terminating at the large dot, their final transformation in the output layer and the resulting prediction with coloured background according to the network's classification.}
    \label{fig:width_squares}
\end{figure}

Indeed, this becomes evident when testing \verb|RK4Net| models of width $\hat{d}=2$ on the datasets \verb|donut_2D| and \verb|squares|. Due to their specific topological properties, the features belonging to data points of one class need to be squeezed through a gap between the features corresponding to points of the other class in order to be separated by a line. In the case of \verb|donut_2D|, the blue points are pushed out of the center through a narrow path in the ring as depicted in the top row of Figure~\ref{fig:width_donut_2D}. For \verb|squares|, the blue points in the upper right corner are, for instance, channeled through the center where the four squares meet illustrated in the top row of Figure~\ref{fig:width_squares}. This results in rather complicated trajectories which are computational costly and difficult to learn. Furthermore, the prediction accuracy on these sets is always limited since the prediction boundary takes an unfavourable shape as illustrated in the top right plots of Figure \ref{fig:width_donut_2D} and Figure \ref{fig:width_squares}. 

If we augment the original two dimensional space by just one dimension, that is $\hat{d}=3$, the performance on both of these problematic datasets improves significantly. First of all, the trajectories of the features in 3D become much simpler because they can be moved in opposing directions along the added dimension. The bottom trajectories and transformation plots of Figure \ref{fig:width_donut_2D} show how the blue points in the center of \verb|donut_2D| are pushed upwards forming the top of a cone. Similarly, the bottom trajectories and transformation plots of Figure \ref{fig:width_squares} illustrate how the blue points of \verb|squares| are pushed downwards while the red ones are moved upwards so that the squares in which the data points are located get deformed. As a result, the features become effortlessly separable by a hyperplane. Besides that, the prediction can reach a higher accuracy with a separating line that coincides with the intuitive borderline between the points of the two classes as in the bottom right plots in Figure \ref{fig:width_donut_2D} and Figure \ref{fig:width_squares}. 
\if{
\begin{figure}
\subfigure[Bild 1]{\includegraphics[width=0.25\textwidth]{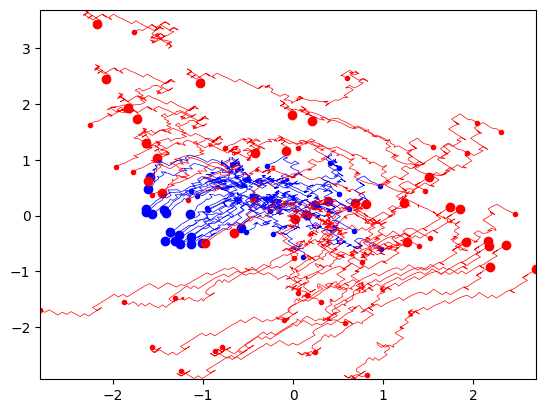}}
\subfigure[Bild 2]{\includegraphics[width=0.25\textwidth]{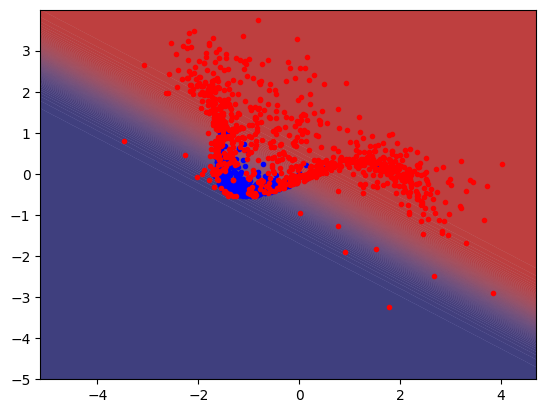}}
\subfigure[Bild 3]{
\includegraphics[width=0.25\textwidth]{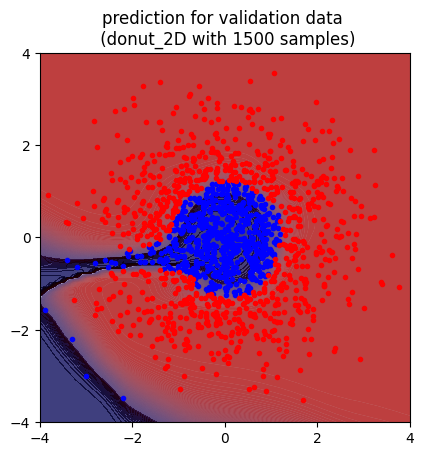}}
\caption[Classification of \texttt{donut\_2D} with \texttt{RK4Net} based on NODEs or ANODEs]{Classification of \texttt{donut\_2D} with \texttt{RK4Net} of width $\hat{d}=2$ corresponding to the NODE-approach (top) and $\hat{d}=3$, i.\,e.\ with space augmentation characterizing the ANODE-approach (bottom), and of same depth  $L=100$ and $\tanh$ activation. The plots show (from left to right) the trajectories of the features starting at the small dot and terminating at the large dot, their final transformation in the output layer and the resulting prediction with coloured background according to the network's classification.}
    \label{fig:width_donut_2D}
\end{figure}
}\fi

This approach of space augmentation leading to ANODEs was proposed in \cite{dupont2019ANODE}. Since augmenting the feature space  generally yields better results across different datasets and initializations, we will always follow that approach in the following experiments by ensuring that $\hat{d}>d$, i.\,e.\ by choosing the width larger than the dimension of the input space.


\subsection{Experiments on network depth}

\label{sec:depth}

Theoretically, increasing the depth of a network architecture improves its expressivity and therefore its performance significantly, see  \cite[pp.198--200]{goodfellow2016DL}. At the same time, it might be harder to train deeper network models and it is questionable whether the optimization algorithm is able to find suitable parameter values at all. Indeed, the experiments in \cite{he2015DResLearningforImage} show that a degradation process can occur: With the depth increasing, both the training and the validation accuracy initially rises and then declines rapidly. Note that this sudden drop is not caused by overfitting because then the training accuracy would stay high. He et al. \cite{he2015DResLearningforImage} observed this phenomenon in feed-forward neural networks without residual connections. However, when such skip connections were added, the network model remained unaffected by this degradation problem. Moreover, the so constructed ResNets could even benefit from greater depth by gaining accuracy. 
 
The question arises whether the same effect occurs in our setting and, particularly, whether RK Nets are affected by this degradation of accuracy. In order to answer that, we will perform an experiment on increasing the depth of the baseline network \verb|StandardNet| and of \verb|RK4Net| as a representative of RK Nets. For that, we select the two dimensional dataset \verb|spiral| with two classes depicted in the right plot of Figure \ref{fig:data_examples}. We train the network models in a generously augmented space of dimensionality $\hat{d}=16$ using the $\tanh$ activation function. The depth $L$ of both networks is then gradually increased from only one layer to maximal 100 layers. Furthermore, each experimental configuration is run four times, so that the evaluation of the results can be based on multiple initializations. Table \ref{tab:depth_accuracy} shows the mean of the accuracy over these repetitions. 

\begin{table}
\centering
\begin{tabular}{ |c|c|c|c|c|c|c|c| } 
 \hline
 depth $L$ & 1 & 3 & 5 & 10 & 20 & 40 & 100\\
 \hline
 \hline
 \multirow{2}{*}{\texttt{StandardNet}} 
 & 92.73 & 92.87 & 98.12 & 97.52 & 67.62 & 51.08 & 50.67\\
 & 91.88 & 92.50 & 98.10 & 97.45 & 66.87 & 48.92 & 49.33\\
 \hline
 \multirow{2}{*}{\texttt{RK4Net}}
 & 75.60 & 91.42 & 97.90 & 99.77 & 99.93 & 99.73 & 99.95\\
 & 75.12 & 90.68 & 97.33 & 99.47 & 99.70 & 99.50 & 99.75\\
 \hline
\end{tabular}
\caption[Mean of training and validation accuracy on \texttt{spiral}]{Mean of training (upper row) and validation (lower row) accuracy (\%) over four repetitions on \texttt{spiral} with network width $\hat{d}=16$ and $\tanh$ activation.}
\label{tab:depth_accuracy}
\end{table}

\begin{table}
\centering
\begin{tabular}{ |c|c|c|c|c|c|c|c| } 
 \hline
 depth $L$ & 1 & 3 & 5 & 10 & 20 & 40 & 100\\
 \hline
 \hline
 \multirow{2}{*}{\texttt{StandardNet}} 
 & 2.23 & 1.38 & 0.66 & 0.77 & 6.09 & 6.93 & 6.93\\
 & 2.33 & 1.53 & 0.67 & 0.77 & 6.13 & 6.94 & 6.93\\
 \hline
 \multirow{2}{*}{\texttt{RK4Net}}
 & 4.32 & 2.68 & 0.98 & 0.16 & 0.04 & 0.10 & 0.01\\
 & 4.39 & 2.69 & 1.06 & 0.28 & 0.13 & 0.12 & 0.12\\
 \hline
\end{tabular}
\caption[Mean of training and validation cost on \texttt{spiral}]{Mean of training (upper row) and validation (lower row) cost ($\times 10^{-1}$) over four repetitions on \texttt{spiral} with network width $\hat{d}=16$ and $\tanh$ activation.}
\label{tab:depth_cost}
\end{table}

Unsurprisingly, the training accuracy for both networks is consistently higher than the validation accuracy, but since this gap is negligible, no overfitting takes place, even without using a regularizer. Nevertheless, the prediction accuracy  of each network architecture greatly varies across different depths. Clearly, we can confirm that the degradation process happens in the simple feed-forward \verb|StandardNet|. Although we observe an improvement when increasing the depth from one to five layers, the accuracy starts declining with a depth of 10 and drops down significantly when reaching 20 layers. Eventually, deep \verb|StandardNet| models with 40 or more layers are not performing better than random class assignment. Interestingly, this degradation problem does not occur in the RK Net \verb|RK4Net|. On the contrary, the accuracy gradually rises with depth. This increase is particularly strong in the shallow models from one to about 10 layers. Then, the training as well as the validation accuracy is already above 99\,\% and is not changing much when deepening the models up to 100 layers. Comparing both network architectures, we see that a shallow \verb|StandardNet| outperforms a shallow \verb|RK4Net|, but when considering deep architectures, the \verb|RK4Net| achieves overall the best classification results. If we repeat this experiment with networks augmented by only one additional dimension, i.\,e.\ with $\hat{d}=3$, the accuracy of the \verb|StandardNet| never exceeds 80\,\% and degrades as before, though starting at a slightly higher number of layers. In contrast, \verb|RK4Net| still manages to classify nearly perfectly but requires 40 or more layers.

The same conclusion can be drawn from the values of the cost function displayed in Table \ref{tab:depth_cost}. As expected, the cost on the validation data is slightly higher than on the training data. Furthermore, the development of the cost with respect to the depth is exactly reversed to the development of the accuracy. These observations accord with the previous results since a high accuracy is linked to a low cost via the used loss function and vice versa. 

\begin{figure}
    \centering
    \setlength\tabcolsep{0pt}
    \settowidth\rotheadsize{\texttt{StandardNet}}
    \begin{tabularx}{\linewidth}{l cccccc}
        & {\small input} &  &  &  &  & {\small output}\\ 
        \rothead{\centering \texttt{StandardNet}}       
        & \includegraphics[trim={0.5cm 0 0.5cm 0}, clip, width=2.1cm,valign=m]{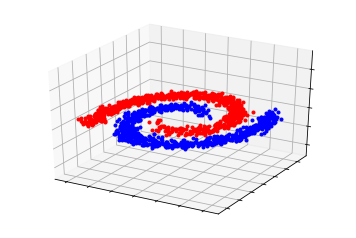}
        & \includegraphics[trim={0.5cm 0 0.5cm 0}, clip, width=2.1cm,valign=m]{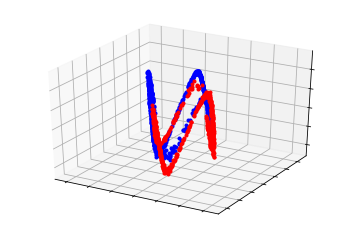}
        &\includegraphics[trim={0.5cm 0 0.5cm 0}, clip, width=2.1cm,valign=m]{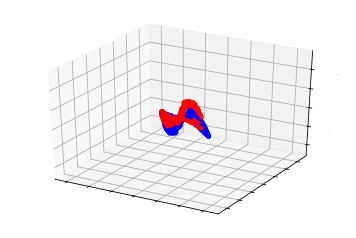}
        & \includegraphics[trim={0.5cm 0 0.5cm 0}, clip, width=2.1cm,valign=m]{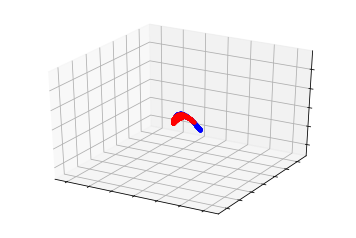}
        &\includegraphics[trim={0.5cm 0 0.5cm 0}, clip, width=2.1cm,valign=m]{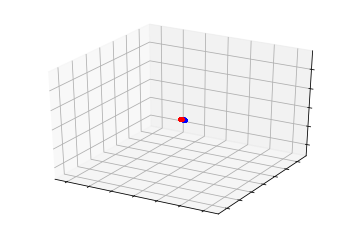}
        &\includegraphics[trim={0.5cm 0 0.5cm 0}, clip, width=2.1cm,valign=m]{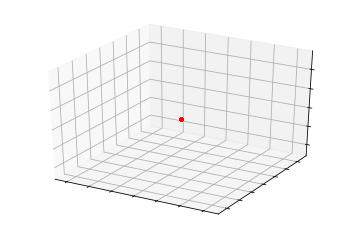}\\ 
        \rothead{\centering \texttt{RK4Net}} 
        & \includegraphics[trim={0.5cm 0 0.5cm 0}, clip, width=2.1cm,valign=m]{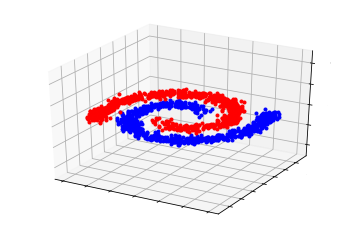}
        & \includegraphics[trim={0.5cm 0 0.5cm 0}, clip, width=2.1cm,valign=m]{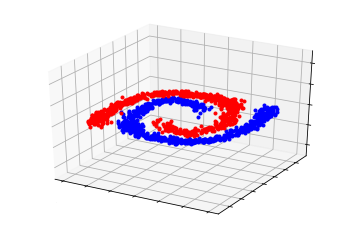}
        &\includegraphics[trim={0.5cm 0 0.5cm 0}, clip, width=2.1cm,valign=m]{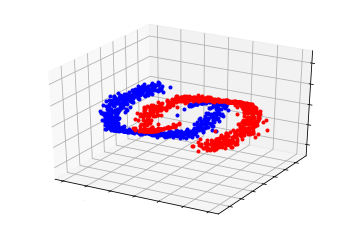}
        & \includegraphics[trim={0.5cm 0 0.5cm 0}, clip, width=2.1cm,valign=m]{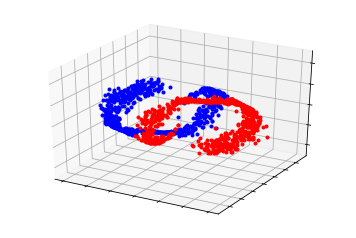}
        &\includegraphics[trim={0.5cm 0 0.5cm 0}, clip, width=2.1cm,valign=m]{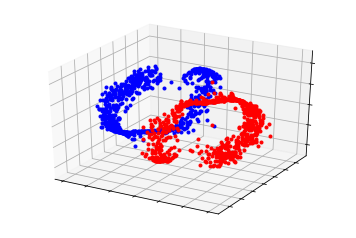}
        &\includegraphics[trim={0.5cm 0 0.5cm 0}, clip, width=2.1cm,valign=m]{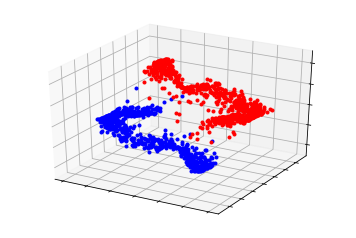}
    \end{tabularx}
    \caption[Feature transformation of \texttt{spiral} with \texttt{StandardNet} and \texttt{RK4Net}]{Feature transformation of \texttt{spiral} with \texttt{StandardNet} (top) and \texttt{RK4Net} (bottom) of width $\hat{d}=16$, depth $L=20$ and $\tanh$ activation. (From left to right) features in input layer, hidden layers and output layer.}
    \label{fig:depth_transformation}
\end{figure}

To better understand this behaviour, we analyze the transformation of the features when passed through the neural network. In the following, we consider a network depth of 20 layers since this is essentially the turning point where the \verb|StandardNet| breaks down while the \verb|RK4Net| reaches its excellent performance. In order to reduce the dimensionality of the feature space to 3D, PCA is employed before visualization. Then, a sequence of transformation plots can be created for each network architecture as depicted in Figure \ref{fig:depth_transformation}. The features are displayed in the input layer, in select hidden layers and finally in the output layer, so that their evolution becomes visible. Starting at the original spiral lifted into the augmented space, the features within the \verb|StandardNet| model initially extend into the direction of the added dimensions, but are then rapidly compressed to the zero vector of the feature space. As the features vanish, it is impossible to distinguish dots belonging to different classes. That explains the low prediction quality of this network. The reason for the contraction during the transformation in deeper layers could be the repeated multiplication by small weights. Besides that, the contracting effect of the activation function might play a significant role. This speculation can be confirmed when undertaking the same experiment with the logistic function instead of the hyperbolic tangent function which causes the feature vectors to shrink at an even earlier layer. When residual connections are present, allowing the data to jump over layers, this kind of information loss can be prevented. For that reason, the \verb|RK4Net| model does not exhibit such an unfavorable feature transformation. The blue and red dots of the spiral are progressively pulled apart, such that a hyperplane can easily separate the feature vectors of the two classes in the output layer. This indicates that the degradation problem is well addressed by RK Nets. In addition, we manage to obtain an even clearer feature separation by increasing the depth.


\subsection{Experiments on network activation}

\label{sec:activation}

Last but not least, we consider how different activation functions affect the performance of RK Nets. The respective experiments are done with the \verb|RK4Net| on the \verb|donut_1D| dataset depicted in the left plot of Figure \ref{fig:data_examples}. 

First of all, we observe that the shape of the prediction boundary does not depend on the choice of the activation function since all successful classifications look very similar to the one illustrated in Figure \ref{fig:act_pred}. Secondly, we find that all available activations lead to the same prediction accuracy of roughly 90\,\% without overfitting, given an augmented and sufficiently deep network architecture. However, the hyperbolic tangent needs less layers than the logistic function to classify correctly. Besides that, $\tanh$  converges significantly faster than $\logistic$ since it has a stronger gradient. As expected, $\ReLU$ and the smoothed $\softplus$ function yield very similar results. Both exhibit a good prediction accuracy even for fairly shallow models and converge after only five epochs. These convergence rates can be deduced from Figure \ref{fig:act_conv}. When repeating this experiment on well-performing shallow \verb|StandardNet| models, we detect the same behaviour with respect to convergence. In conclusion, this influence of the activation functions on the training progress is not unique to RK Nets, but seems to generalize to all kinds of network architectures.

\begin{figure}
    \centering
    \includegraphics[trim={0 0 0 1.255cm}, clip, width=4.5cm]{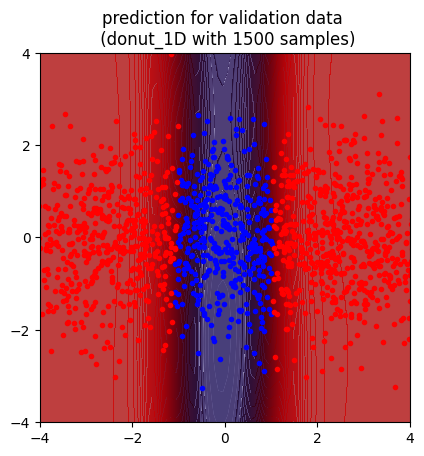}
    \caption[Prediction of \texttt{donut\_1D} with \texttt{RK4Net}]{Prediction of \texttt{donut\_1D} with \texttt{RK4Net} of width $\hat{d}=16$, depth $L=20$ and $\tanh$ activation.}
    \label{fig:act_pred}
\end{figure}

\begin{figure}
    \centering
    \includegraphics[width=6.2cm]{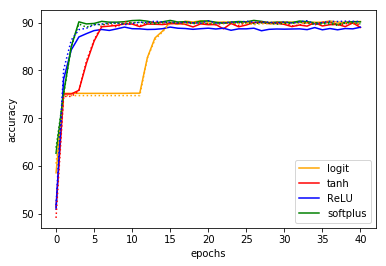}
    \includegraphics[width=6.2cm]{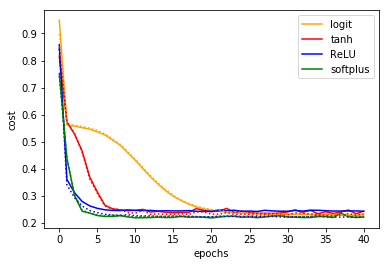}
    \caption[Accuracy and cost on \texttt{donut\_1D} with \texttt{RK4Net}]{Accuracy (left) and cost (right) over the course of epochs on \texttt{donut\_1D} with \texttt{RK4Net} of width $\hat{d}=16$ and depth $L=20$. Solid lines represent metrics on validation and dotted lines on training data.}
    \label{fig:act_conv}
\end{figure}

\subsection{Comparison of Runge-Kutta Nets to standard network architecture}

After identifying good settings for the main hyperparameters separately for each network architechture, we seek to investigate how RK Nets differ from feed-forward networks. In particular, we are interested in the fact whether RK Nets are superior to these conventional network architectures in practice. Therefore, we undertake an experiment in which we compare the performance of \verb|EulerNet| and \verb|RK4Net| to that of \verb|StandardNet| on the more complicated point datasets. More precisely, we classify several \verb|donut_multiclass| and \verb|squares_multiclass| sets whose dimensionality and number of classes can be varied. Thus, we are able to examine the effect of shifting from binary to multiclass classification by adding another class, and of increasing the input dimension from 2D to 3D. All of the tested datasets are depicted in Figure \ref{fig:comparison_datasets}. 

\begin{figure}
    \centering
    \begin{tabular}{ ccc }
    {\small donut 3D \& 2C} & {\small donut 2D \& 6C} & {\small squares 2D \& 4C}\\
    \includegraphics[trim={0 0 0 1.255cm}, clip, width=4cm]{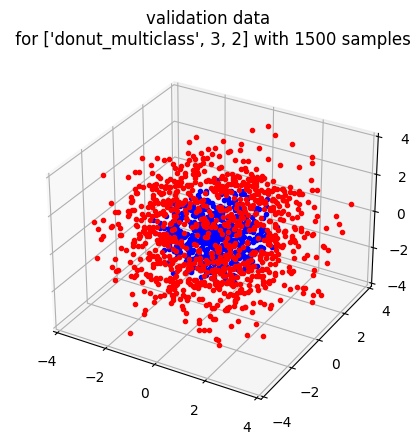}
    & \includegraphics[trim={0 0 0 1.255cm}, clip, width=4cm]{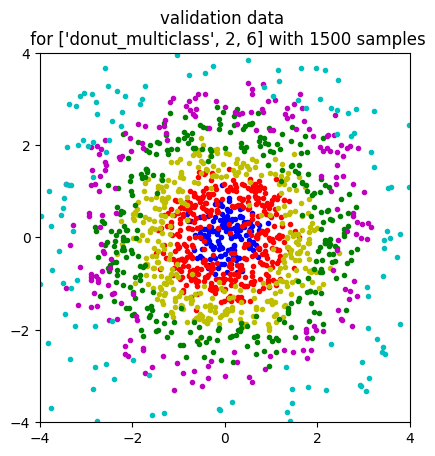}
    & \includegraphics[trim={0 0 0 1.255cm}, clip, width=4cm]{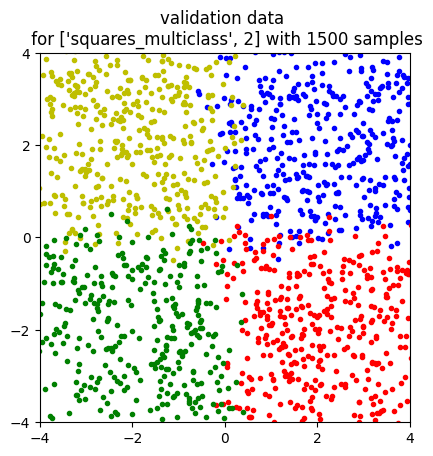}\\
    \\[0.5ex]
    {\small donut 3D \& 3C} & {\small donut 3D \& 6C} & {\small squares 3D \& 4C}\\
    \includegraphics[trim={0 0 0 1.255cm}, clip, width=4cm]{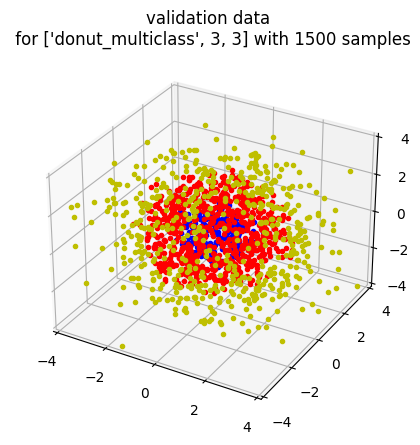}
    & \includegraphics[trim={0 0 0 1.255cm}, clip, width=4cm]{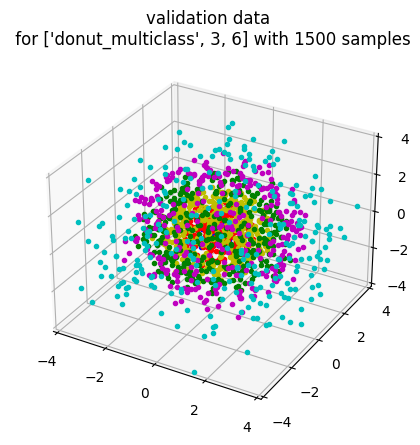}
    & \includegraphics[trim={0 0 0 1.255cm}, clip, width=4cm]{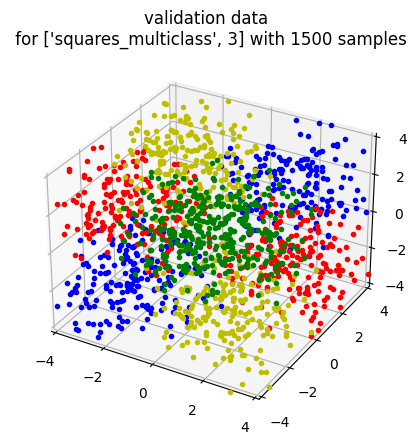}
    \end{tabular}
    \caption[Donut and squares datasets of different dimensionality and number of classes]{Donut and squares datasets of different dimensionality and with varying number of classes used for comparing performance of networks between binary and multiclass classification (first column), as well as 2D and 3D input space (second and third column).}
    \label{fig:comparison_datasets}
\end{figure}

As we discovered in Section \ref{sec:width}, an augmented feature space is generally advantageous, so we choose a relatively large width of $\hat{d}=16$. Since RK Nets benefit from a great number of layers as found in Section \ref{sec:depth}, we set the depth to $L=100$ for both \verb|EulerNet| and \verb|RK4Net|. To establish a fair comparison to the simple feed-forward network whose performance degrades with increasing depth, we decide on a shallow \verb|StandardNet| with $L=5$. Furthermore, we use the hyperbolic tangent activation function for all network architectures.

Before comparing the selected network architectures with each other, we test the sensitivity of the numerical results to initializations within one model. More precisely, we seek to analyze to what extent the behaviour of a neural network depends on the random sampling of the data points and the initialization of the trainable network parameters, i.\,e.\ the weights and biases. For conciseness, we only show the results for \verb|RK4Net| on the two dimensional donut dataset with six classes, but remark that the observations are similar across all network architectures and datasets. As a measure for the variability of the results, we consider the standard deviation of the accuracy and cost over four experiment runs. As displayed in Table \ref{tab:sensitivity}, the variability is negligible both on the training and the validation data. Therefore, the network performance seems to be robust to random initializations. This finding is also reflected in the visualization of the transformed data and the prediction of each repetition shown in Figure \ref{fig:sensitivity}. Indeed, the data points are transformed into a very similar shape and the prediction boundaries between the classes are nearly the same. Consequently, it is justified to base the following evaluation of the different networks on averages over four initializations.

\begin{table}
\centering
\begin{tabular}{ |c||c|c||c|c| } 
\hline
& training & validation & training & validation\\
& accuracy & accuracy & cost & cost\\
\hline
\hline
mean & 77.13 & 74.92 & 5.13 & 5.59\\
\hline
standard deviation & 0.76 & 0.89 & 0.08 & 0.16\\
\hline
\end{tabular}
\caption[Sensitivity]{Variability of accuracy (\%) and cost ($\times 10^{-1}$) over four repetitions for \texttt{RK4Net} with width $\hat{d}=16$, depth $L=100$ and $\tanh$ activation, on donut 2D \& 6C.}
\label{tab:sensitivity}
\end{table}

\begin{figure}
    \centering
    \setlength\tabcolsep{0pt}
    \settowidth\rotheadsize{transformation (3D)}
    \begin{tabularx}{\linewidth}{l cccc}
        & {\small repetition 1} & {\small repetition 2} & {\small repetition 3} & {\small repetition 4}\\ 
        \rothead{\centering transformation (3D)}       
        & \includegraphics[width=3.1cm,valign=m]{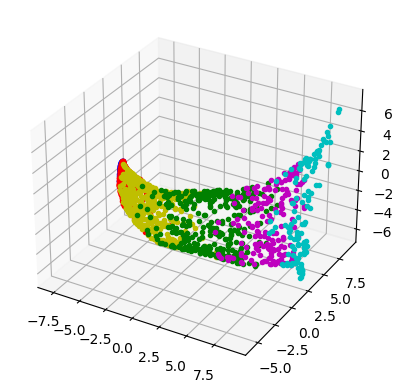}
        & \includegraphics[width=3.1cm,valign=m]{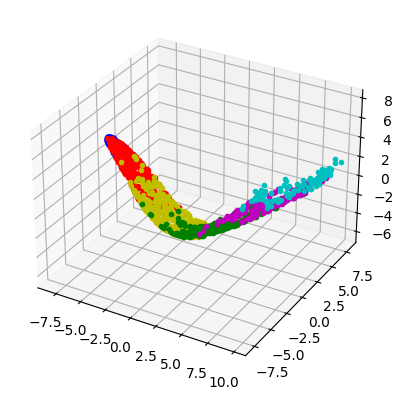}
        & \includegraphics[width=3.1cm,valign=m]{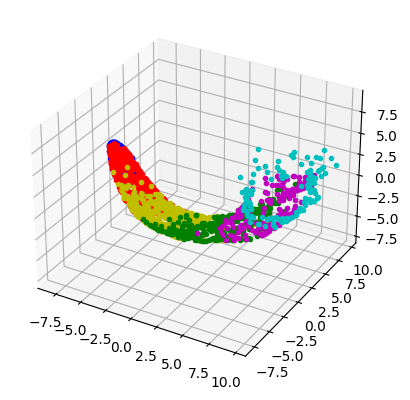}
        & \includegraphics[width=3.1cm,valign=m]{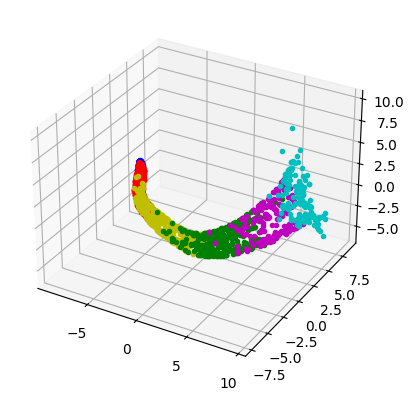}\\ 
        \addlinespace[2pt]
        \rothead{\centering prediction} 
        &\includegraphics[trim={0 0 0 1.255cm}, clip, width=3.1cm,valign=m]{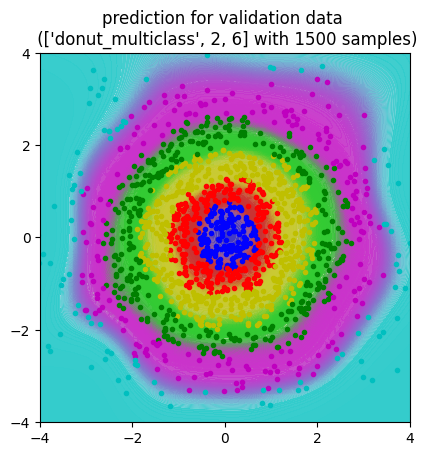}
        &\includegraphics[trim={0 0 0 1.255cm}, clip, width=3.1cm,valign=m]{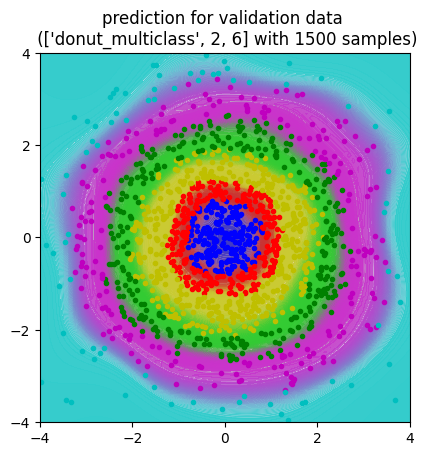}
        &\includegraphics[trim={0 0 0 1.255cm}, clip, width=3.1cm,valign=m]{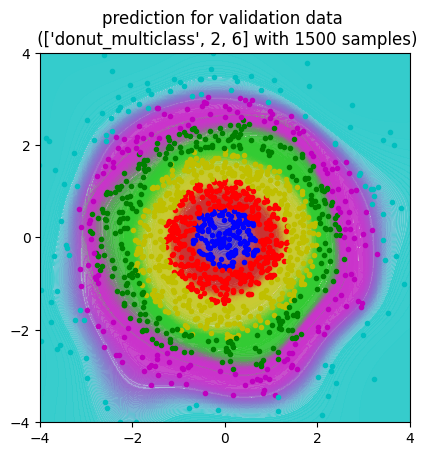}
        &\includegraphics[trim={0 0 0 1.255cm}, clip, width=3.1cm,valign=m]{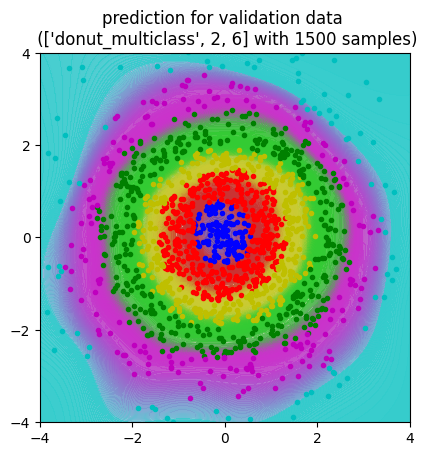}
    \end{tabularx}
    \caption[Repetitions with random initializations]{Repetitions with random initializations for \texttt{RK4Net} with width $\hat{d}=16$, depth $L=100$ and $\tanh$ activation, on donut 2D \& 6C. The plots show (upper row) the feature transformation in the output layer reduced by PCA to 3D, and (lower row) the resulting prediction underlaid with a coloured background according to the network's classification.}
    \label{fig:sensitivity}
\end{figure}

We start by comparing the validation metrics, namely accuracy and cost, of models built according to different network architectures and trained on various donut and squares sets. As discussed above, we run each experiment configuration four times to make sure that the results do not depend on randomness. The mean of both metrics over these repetitions are displayed in Table \ref{tab:comparison_metrics}. 

First of all, we note that networks of each architecture manage to classify all validation datasets predominantly correctly. Naturally, this is a consequence of choosing the hyperparameters very carefully. However, the achieved prediction accuracy never exceeds 95\,\% for all datasets. The reason for this is that the datasets include noise, meaning that the boundary between points of distinct classes are rather fuzzy, so that a point of a specific class occasionally falls into the region of the adjacent class. This noise was deliberately added in order to test the robustness of the neural networks. Under these circumstances, the performance of all network architectures are definitely satisfying. Moreover, we see that \verb|StandardNet| does equally well as both RK Nets, \verb|EulerNet| and \verb|RK4Net|. The small differences between the prediction quality of the tested network architectures originate solely from the variance in between the experiment repetitions caused by random initialization which we confirmed by checking the standard deviation over these repeated measurements.

Furthermore, we find that a dataset with one additional class or with one additional dimension is generally a little harder to classify. This becomes evident when comparing the metrics of the donuts in 3D with either two or three classes, as well as of the two and three dimensional donuts of each 6 classes or the two and three dimensional squares of each 4 classes with one another: The accuracy of the respective latter set is consistently smaller coupled with a slightly higher value of the cost. This observation coincides with our expectation of how to rate the difficulty of a classification task.

\begin{table}
\centering
\begin{tabular}{ |c||c|c||c|c||c|c| } 
\hline
& donut & donut & donut & donut & squares & squares\\
& 3D \& 2C & 3D \& 3C & 2D \& 6C & 3D \& 6C & 2D \& 4C & 3D \& 4C\\
\hline
\hline
\multirow{2}{*}{\texttt{StandardNet}} 
& 92.37  & 87.75 & 75.12 & 73.00 & 94.12 & 89.68\\
& 1.71 & 2.85 & 5.60 & 5.86 & 1.57 & 3.03\\
\hline
\multirow{2}{*}{\texttt{EulerNet}} 
& 91.88 & 88.30 & 74.87 & 74.63 & 93.35 & 89.48\\
& 1.84 & 2.75 & 5.56 & 5.67 & 1.66 & 2.71\\
\hline
\multirow{2}{*}{\texttt{RK4Net}} 
& 92.73 & 87.13 & 74.92 & 74.88 & 93.20 & 89.37\\
& 1.72 & 2.95 & 5.59 & 5.73 & 1.64 & 2.81\\
\hline
\end{tabular}
\caption[Mean of validation accuracy and cost]{Mean of validation accuracy (\%, upper row) and cost ($\times 10^{-1}$, lower row) over four repetitions with network width $\hat{d}=16$, depth $L=5$ for \texttt{StandardNet} and $L=100$ for \texttt{EulerNet} and \texttt{RK4Net}, and $\tanh$ activation.}
\label{tab:comparison_metrics}
\end{table}

In order to develop a better understanding of how the network models reach their prediction, we consider the feature transformation within the network layers, particularly in the output layer. Since the networks are equipped with a 16 dimensional feature space, we firstly need to reduce the dimensionalty to 2D or 3D, so that we can visualize the feature vectors. As before, this is realized by PCA. Comparing these transformation plots across several repetitions, we find that the shapes and patterns are very similar. Moreover, the transformations of all donut sets with a particular network architecture have key characteristics in common, regardless of the donut's dimensionality and number of classes. For instance, the \verb|StandardNet| model arranges all features in the form of a string which changes color along its length, as illustrated in the upper transformation plots of Figure \ref{fig:comparison_donut_trans}. Then, the points can be classified by dividing this string into sections of the same color. The transformations in \verb|EulerNet| and \verb|RK4Net| depicted in the middle and lower row of Figure~\ref{fig:comparison_donut_trans} strongly resemble each other: The features evolve to a cone with the color transitioning along its height. So, dots of different colors can be separated from each other by slicing the cone. This suggests that all RK Nets share a similar feature transformation determined by the ODE they are derived from. However, the underlying dynamics of standard networks and RK Nets seem to be fundamentally distinct from each other as their dissimilar feature evolution indicates. The same holds true for squares sets of different dimensionality and class numbers. Here, \verb|StandardNet| has the shape of a closed curve, with all dots of one color allocated in one segment, whereas \verb|EulerNet| and \verb|RK4Net| create a wavy surface while pushing the dots to the outer corner of their original square. This is shown in the left and middle column of Figure \ref{fig:comparison_squares_trans}.

Despite their different feature transformations, the predictions of standard feed-forward network and RK Net models look surprisingly similar. For visualization, we choose the two dimensional donut and squares dataset with six and four classes, respectively, since plotting in 2D offers the possibility of adding a coloured background according to the network's classification of the entire input space. When comparing the shape of the devision lines in the prediction plots in the right columns of Figure \ref{fig:comparison_donut_trans} and \ref{fig:comparison_squares_trans}, we see that they are almost indistinguishable. That also explains why models of all network architectures achieve approximately the same accuracy for a particular dataset.

\begin{figure}
    \centering
    \setlength\tabcolsep{0pt}
    \settowidth\rotheadsize{\texttt{StandardNet}}
    \begin{tabularx}{\linewidth}{l ccc}
        & {\small transformation (3D)} & {\small transformation (2D)} & {\small prediction}\\ 
        \rothead{\centering \texttt{StandardNet}}       
        & \includegraphics[width=4.1cm,valign=m]{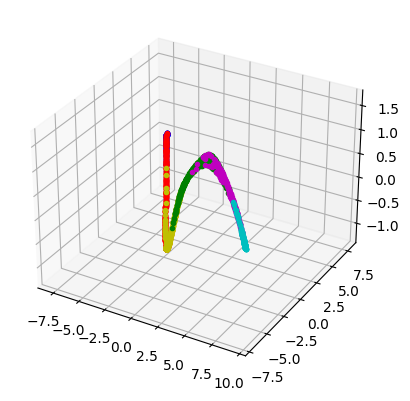}
        & \includegraphics[width=4.1cm,valign=m]{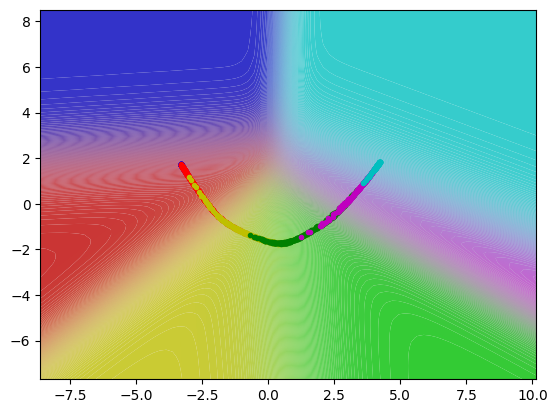}
        &\includegraphics[trim={0 0 0 1.255cm}, clip, width=4.1cm,valign=m]{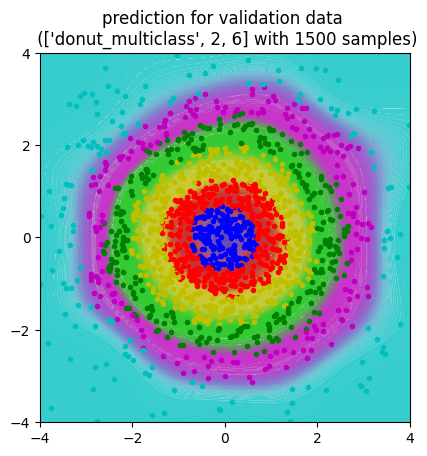}\\ 
        \addlinespace[2pt]
        \rothead{\centering \texttt{EulerNet}} 
        & \includegraphics[width=4.1cm,valign=m]{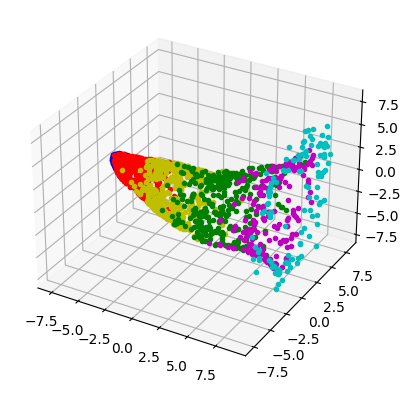}
        & \includegraphics[width=4.1cm,valign=m]{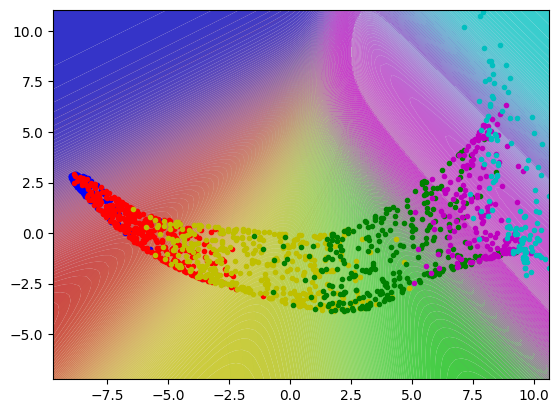}
        &\includegraphics[trim={0 0 0 1.255cm}, clip, width=4.1cm,valign=m]{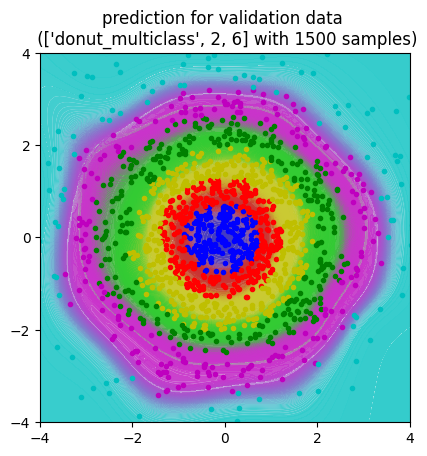}\\
        \addlinespace[2pt]
        \rothead{\centering \texttt{RK4Net}} 
        & \includegraphics[width=4.1cm,valign=m]{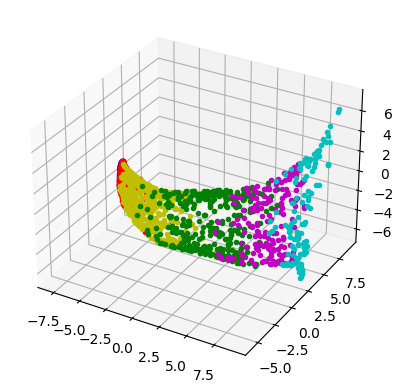}
        & \includegraphics[width=4.1cm,valign=m]{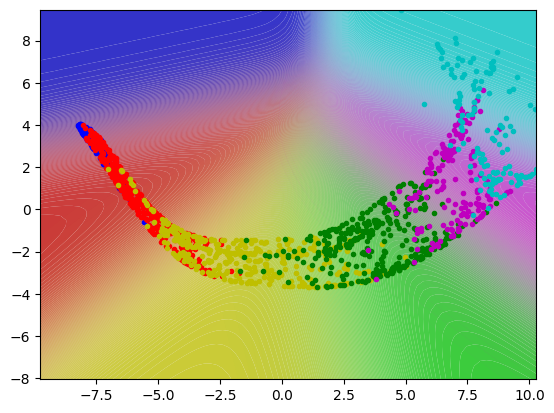}
        &\includegraphics[trim={0 0 0 1.255cm}, clip, width=4.1cm,valign=m]{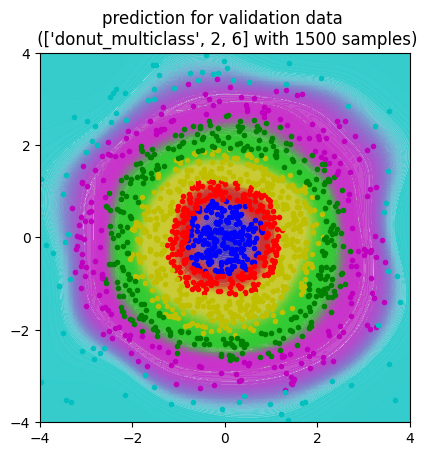}
    \end{tabularx}
    \caption[Classification of donut 2D \& 6C]{Classification of donut 2D \& 6C with network width $\hat{d}=16$, depth $L=5$ for \texttt{StandardNet} and $L=100$ for \texttt{EulerNet} and \texttt{RK4Net}, and $\tanh$ activation. The plots show (from left to right) the feature transformation in the output layer reduced by PCA to 3D and 2D, and the resulting prediction. Two dimensional plots are underlaid with a coloured background according to the network's classification.}
    \label{fig:comparison_donut_trans}
\end{figure}

\begin{figure}
    \centering
    \setlength\tabcolsep{0pt}
    \settowidth\rotheadsize{\texttt{StandardNet}}
    \begin{tabularx}{\linewidth}{l ccc}
        & {\small transformation (3D)} & {\small transformation (2D)} & {\small prediction}\\ 
        \rothead{\centering \texttt{StandardNet}}       
        & \includegraphics[width=4.1cm,valign=m]{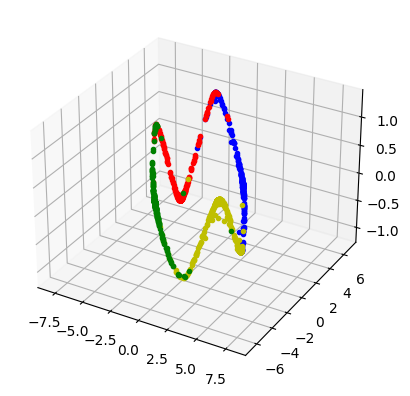}
        & \includegraphics[width=4.1cm,valign=m]{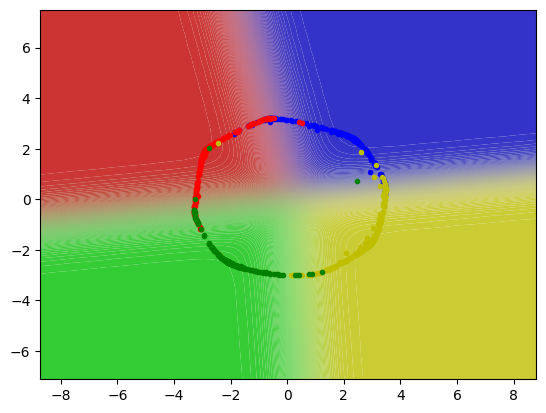}
        &\includegraphics[trim={0 0 0 1.255cm}, clip, width=4.1cm,valign=m]{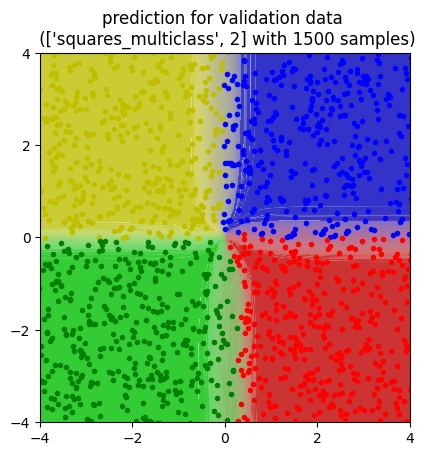}\\
        \addlinespace[2pt]
        \rothead{\centering \texttt{EulerNet}} 
        & \includegraphics[width=4.1cm,valign=m]{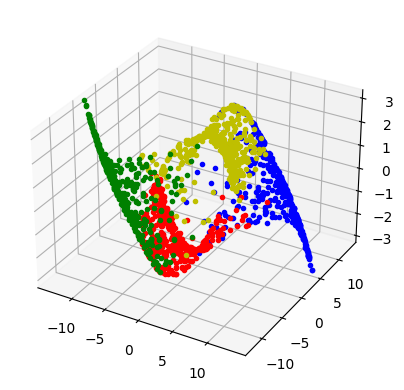}
        & \includegraphics[width=4.1cm,valign=m]{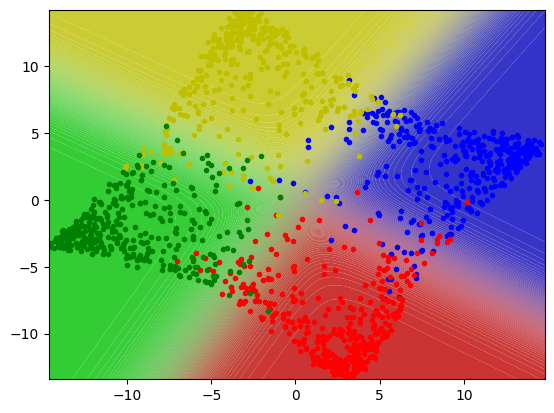}
        &\includegraphics[trim={0 0 0 1.255cm}, clip, width=4.1cm,valign=m]{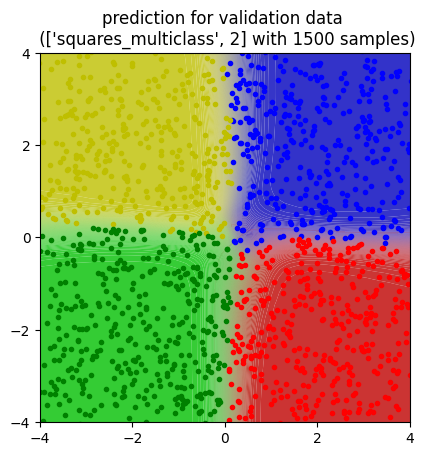}\\
        \addlinespace[2pt]
        \rothead{\centering \texttt{RK4Net}} 
        & \includegraphics[width=4.1cm,valign=m]{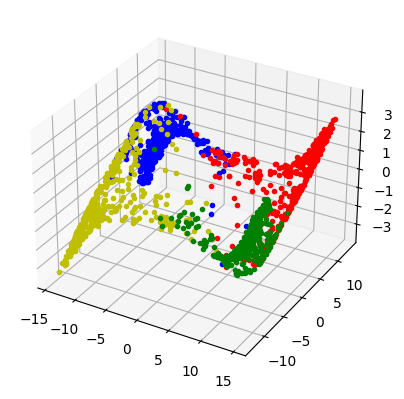}
        & \includegraphics[width=4.1cm,valign=m]{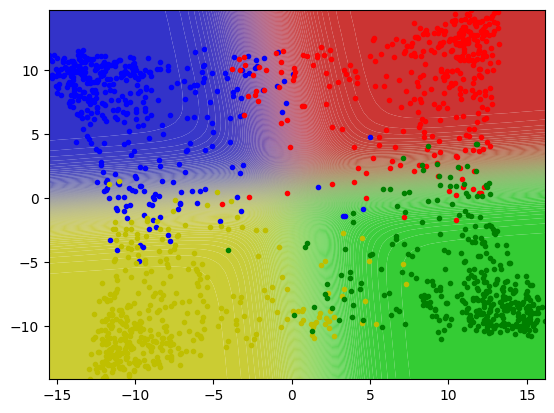}
        &\includegraphics[trim={0 0 0 1.255cm}, clip, width=4.1cm,valign=m]{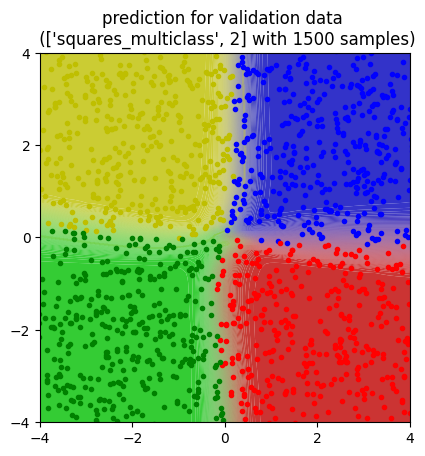}
    \end{tabularx}
    \caption[Classification of squares 2D \& 4C]{Classification of squares 2D \& 4C with network width $\hat{d}=16$, depth $L=5$ for \texttt{StandardNet} and $L=100$ for \texttt{EulerNet} and \texttt{RK4Net}, and $\tanh$ activation. The plots show (from left to right) the feature transformation in the output layer reduced by PCA to 3D and 2D, and the resulting prediction. Two dimensional plots are underlaid with a coloured background according to the network's classification.}
    \label{fig:comparison_squares_trans}
\end{figure}

Finally, the differences between the network architectures with respect to the training process are analyzed. For that, we consider the most challenging datasets among the above selection, which are the three dimensional donut and squares set with four and six classes, respectively. The evolution of the prediction accuracy and the value of the cost function on the validation data are plotted in Figure \ref{fig:comparison_donut_conv} and \ref{fig:comparison_squares_conv}. Here, only the first 14 epochs are displayed, although all network models were trained over up to 40 epochs. Hence, we can focus on the iterations where convergence between network architectures differs. When optimizing on the donut dataset, \verb|EulerNet| is fastest but the other two networks are catching up after around 6 epochs. Ultimately, all networks converge to roughly the same accuracy and cost. For the squares set, we see that \verb|EulerNet| and \verb|RK4Net| have very similar convergence graphs. However, \verb|StandardNet| clearly needs more iterations to reach the prediction accuracy of both RK Nets and the value of its cost remains the highest over the entire training phase. This is not due to the network activation as it was the case in Section \ref{sec:activation} because we use the same function in all models. Besides that, the standard network exhibits slightly more variance across experiment repetitions indicated by the fairly wide shadow around the solid line representing the mean. In conclusion, the training of RK Nets seems to be more robust, especially when taking into account that deep standard networks cannot be optimized at all as shown in Section \ref{sec:depth}. 

\begin{figure}
    \centering
    \includegraphics[width=6.2cm]{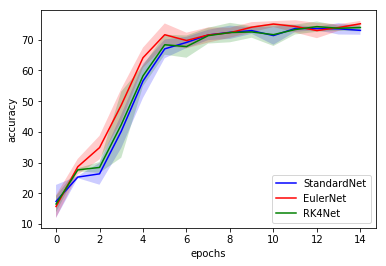}
    \includegraphics[width=6.2cm]{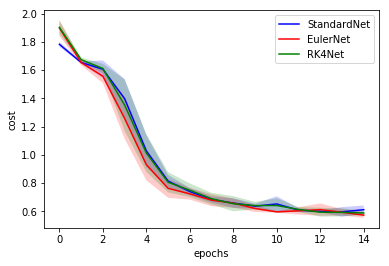}
    \caption[Validation accuracy and cost on donut 3D \& 6C]{Validation accuracy (left) and cost (right) over the course of epochs on donut 3D \& 6C with network width $\hat{d}=16$, depth $L=5$ for \texttt{StandardNet} and $L=100$ for \texttt{EulerNet} and \texttt{RK4Net}, and $\tanh$ activation. Solid line represents the mean and shaded area the standard deviation over repetitions.}
    \label{fig:comparison_donut_conv}
\end{figure}

\begin{figure}
    \centering
    \includegraphics[width=6.2cm]{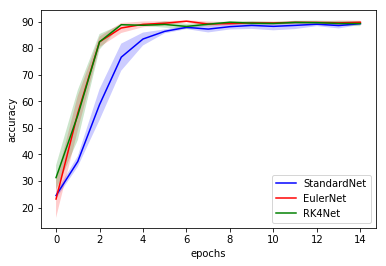}
    \includegraphics[width=6.2cm]{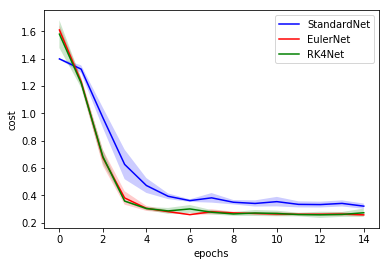}
    \caption[Validation accuracy and cost on squares 3D \& 4C]{Validation accuracy (left) and cost (right) over the course of epochs on squares 3D \& 4C with network width $\hat{d}=16$, depth $L=5$ for \texttt{StandardNet} and $L=100$ for \texttt{EulerNet} and \texttt{RK4Net}, and $\tanh$ activation. Solid line represents the mean and shaded area the standard deviation over repetitions.}
    \label{fig:comparison_squares_conv}
\end{figure}

\subsection{Extension to image classification}

So far, we have only considered point classification tasks. These low-dimensional toy problems are useful for testing new network architectures since they allow for good visualizations which are important for analyzing the behaviour of a network. However, point classification is not an interesting application as such and it remains unclear if the above observations generalize to other more complex datasets. Therefore, we extend our numerical experiments to grayscale images with $28 \times 28$ pixels. Suitable data bases are, for instance, MNIST 
containing handwritten digits and Fashion-MNIST 
which consists of simplified article images, see Figure \ref{fig:MNIST} and \ref{fig:FashionMNIST}. Both of them have ten classes: the labels of the former are given by the digits from zero to nine and the latter is labelled according to the product type such as T-shirt, trouser, pullover or dress.

Typically, image classification is done by using convolution and pooling layers to extract features from the input image before applying a fully-connected network to the flattened feature maps. This approach has the advantage of retaining spatial information since the filters slide over the two-dimensional image and operate only on neighboring pixels contained in the current receptive field. 

We note that there are several ways to transfer the idea of applying Runge-Kutta methods and feature space augmentation to convolutional architectures. In order to use the RK scheme for propagating the image data forward, we only need to make sure that the layer outputs have the same dimension as their inputs. This means that the number of output channels must equal the number of input channels and that the feature map size has to be constant across all layers. While the number of output channels is determined by the number of filters, the height and width of the feature map can be controlled by the filter size, by padding, as well as the stride. Thus, maintaining the original dimension can be achieved by choosing appropriate hyperparameters for convolution. For augmenting the space, we could either increase the spatial size of each image by padding on the borders, resulting in a larger feature map size, or we could add further channels to the input image as proposed in \cite{dupont2019ANODE}. In the case of the considered grayscale pictures, the first option means adding pixels around the $28 \times 28$-sized pixel grid, while the second approach would extend the single channel to multiple channels.

As an alternative to this convolutional network based on Runge-Kutta methods, we will use the same RK Nets as for point classification by reshaping the $28 \times 28$ image as a $28^2$ dimensional vector which is fed directly into the network. Optionally, a few conventional convolution and pooling layers could be added for encoding spatial relations and down-sampling before flattening the data. Such an approach was taken in \cite{chen2018NODE}. However, we omit this step because we are not interested in the effect of convolution. Instead, we solely focus on the previously introduced Runge-Kutta networks with an augmented feature space.

\begin{table}
\centering
\begin{tabular}{ |c||c|c||c|c| } 
 \hline
 \multirow{2}{*}{width $\hat{d}$}
 & training & validation & training & validation\\
 & accuracy & accuracy & cost & cost\\
 \hline
 \hline
 $28^2$& $97.70 \pm 2.80$ & $87.27 \pm 2.93$ & $0.78 \pm 0.95$ & $7.71 \pm 1.62$\\
 \hline
 $30^2$ & $99.77 \pm 0.40$ & $90.40 \pm 1.08$ & $0.10 \pm 0.17$ & $5.36 \pm 0.46$\\
 \hline
\end{tabular}
\caption[]{Mean and standard deviation of accuracy (\%) and cost ($\times 10^{-1}$) over four repetitions for non-augmented (upper row) and augmented (lower row) \texttt{RK4Net} with depth $L=100$ and $\tanh$ activation on MNIST.}
\label{tab:image_width}
\end{table}

For the following experiments the choice of the hyperparameters is based on our findings from point classification. Concerning the width of the network, we can generally observe an improvement when augmenting the space by further dimensions as seen before in the toy examples. To illustrate that, the performance of a non-augmented \verb|RK4Net|, i.\,e.\ $\hat{d}=d=28^2$, on MNIST is compared to its generously augmented version of width $30^2$ in Table \ref{tab:image_width}. Not only is the mean of the accuracy higher and the mean of the cost lower on both the training and the validation set, but also the standard deviation over the four experiment repetitions is significantly smaller in the case of space augmentation. This indicates that the training process is more stable, which could be a result of simpler flows through the network as suggested by the point classification problems above. Besides that, the previous results on network depth also seem to generalize to image classification. While RK Nets are not affected by the degradation process, \verb|StandardNet| suffers heavily from increasing the number of layers up to 100. A deep \verb|StandardNet| does in fact perform not better than random classification since the training and validation accuracy are roughly 10\,\% on image datasets with 10 classes. For that reason, we will choose a \verb|StandardNet| with 5 layers as baseline for evaluating RK Nets of depth 100. Furthermore, we again use the $\tanh$ activation for all architectures.

In order to get a first impression of how RK Nets are behaving with respect to image data, we inspect the predicted label for some example images. Figure \ref{fig:MNIST} and Figure~\ref{fig:FashionMNIST} show a few pictures from the validation datasets of MNIST and Fashion-MNIST along with their true label and the outcome of a deep augmented \verb|RK4Net| trained on the respective data base. Indeed, most images are classified correctly. If we have a closer look at the examples on which the network failed, it is obvious why. For instance, the shape of the digit in the lower right corner of Figure \ref{fig:MNIST} is somewhat ambiguous since the relatively large loop of the six could easily be mistaken as a zero. Similarly, the products shirt and coat, as well as sneaker and sandal share a lot of common features and are therefore hard to distinguish from each other. That explains the wrongly classified items exhibited in Figure \ref{fig:FashionMNIST}. Overall, the quality of the predictions produced by a Runge-Kutta network seems to be fairly good. 

\begin{figure}
    \centering
    \includegraphics[width=8.2cm]{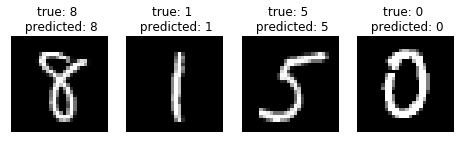}
    \includegraphics[width=8.2cm]{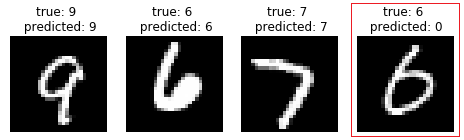}
    \caption[]{Exemplary images of MNIST with true label and prediction produced by \texttt{RK4Net} with width $\hat{d}=30^2$, depth $L=100$ and $\tanh$ activation.}
    \label{fig:MNIST}
\end{figure}

\begin{figure}
    \centering
    \includegraphics[width=8.1cm]{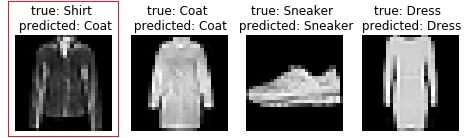}
    \includegraphics[width=8.2cm]{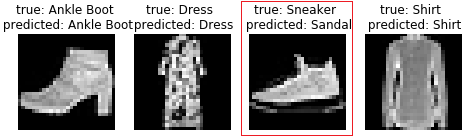}
    \caption[]{Exemplary images of Fashion-MNIST with true label and prediction produced by \texttt{RK4Net} with width $\hat{d}=30^2$, depth $L=100$ and $\tanh$ activation.}
    \label{fig:FashionMNIST}
\end{figure}

We will further investigate the performance of RK Nets by comparing the validation metrics of all network architectures displayed in Table \ref{tab:comparison_metrics_image}. First of all, we find that Fashion-MNIST is generally a more challenging problem than MNIST. This is not very surprising because pictures of articles are more complex and diverse than digits. However, there are also significant differences in the performance of a conventional feed-forward network like \verb|StandardNet| and Runge-Kutta networks such as \verb|EulerNet| and \verb|RK4Net|. The former scores on average about five percentage points less on MNIST and even more than 15 percentage points less on Fashion-MNIST than the latter networks concerning the validation accuracy. As expected, this is coupled with a higher validation cost on both image data bases for \verb|StandardNet|. Moreover, the standard deviation over several experiment runs is comparatively large, particularly for Fashion-MNIST. Consequently, \verb|StandardNet| seems to be less robust to random initialization. Both RK Nets, \verb|EulerNet| and \verb|RK4Net|, provide very similar results and perform clearly better than the baseline. 

\begin{table}
\centering
\begin{tabular}{ |c||c|c|} 
\hline
& MNIST & Fashion-MNIST\\
\hline
\hline
\multirow{2}{*}{\texttt{StandardNet}} 
& $85.67 \pm 0.78$ & $61.23 \pm 6.00$\\
& $8.95 \pm 0.92$ & $11.41 \pm 1.37$\\
\hline
\multirow{2}{*}{\texttt{EulerNet}} 
& $90.98 \pm 0.48$ & $77.62 \pm 2.57$\\
& $5.71 \pm 0.36$ & $9.52 \pm 1.87$\\
\hline
\multirow{2}{*}{\texttt{RK4Net}} 
& $90.40 \pm 1.08$ & $79.13 \pm 1.57$\\
& $5.36 \pm 0.46$ & $8.24 \pm 0.60$\\
\hline
\end{tabular}
\caption[]{Mean and standard deviation of validation accuracy (\%, upper row) and cost ($\times 10^{-1}$, lower row) over four repetitions with network width $\hat{d}=30^2$, depth $L=5$ for \texttt{StandardNet} and $L=100$ for \texttt{EulerNet} and \texttt{RK4Net}, and $\tanh$ activation.}
\label{tab:comparison_metrics_image}
\end{table}

The performance gap between RK Nets and standard network models can be explained by analyzing the transformed features on which the classifier acts. Figure~\ref{fig:image_transformation} depicts the features of the validation data from Fashion-MNIST in the output layer. Since the feature vectors are high dimensional, we reduce their dimension by PCA. Thus, the features of each image can be visualized by a point in a three dimensional space whose color corresponds to its true class. The resulting plot contains point clouds which are ideally clustered by their colors so that points belonging to different classes can be well separated by a classifier. This is indeed the case for the plot on the right-hand side of Figure \ref{fig:image_transformation} showing the features within a \verb|RK4Net| model. However, the features evolving in \verb|StandardNet| plotted on the left-hand side are rather scattered across space with points of the same color being far apart in some cases. Arranging the points in this way makes it hard to classify them correctly.

\begin{figure}
    \centering
    \includegraphics[width=6.2cm]{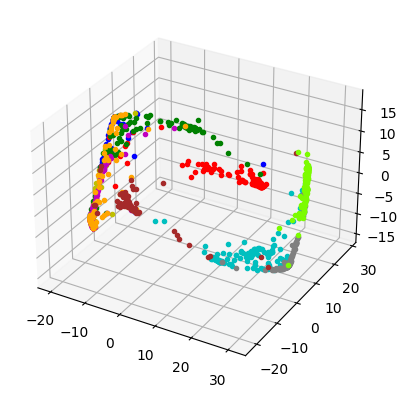}
    \includegraphics[width=6.2cm]{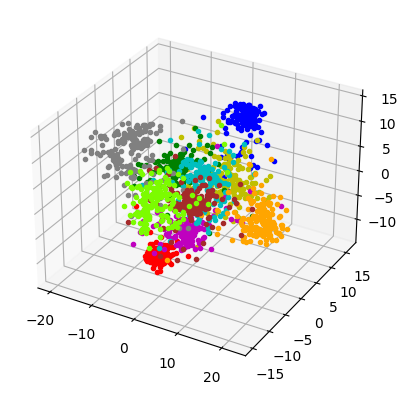}
    \caption[]{Feature transformation in the output layer of \texttt{StandardNet} (left) and \texttt{RK4Net} (right) of Fashion-MNIST images reduced by PCA to 3D. Each color represents one article class.}
    \label{fig:image_transformation}
\end{figure}

Last but not least, we will investigate the convergence of the accuracy and cost during the training phase. For that, the training and validation metrics for all network architectures are plotted over the first 25 epochs in Figure \ref{fig:MNIST_convergence} and Figure \ref{fig:FashionMNIST_convergence} for MNIST and Fashion-MNIST, respectively. Note that the networks were trained even longer for maximum 40 epochs with the possibility of early stopping when the model started to overfit. Firstly, we find that the convergence behaviour of \verb|EulerNet| and \verb|RK4Net| strongly resemble each other. Secondly, both RK Nets clearly outperform \verb|StandardNet|, particularly on the more challenging Fashion-MNIST images. We have already seen before that their final accuracy is higher accompanied with a lower cost, but in addition they also converge faster compared to the baseline model. Furthermore, the convergence graphs of \verb|StandardNet| are relatively erratic, especially towards the end of training on Fashion-MNIST, which indicates instabilities in the optimization process. In contrast, RK Nets seem to be easier to train resulting into a much smoother convergence graph. Since the plots show both the validation and the training values, we can further investigate how the networks generalize to unseen data. Overall, we observe a significant gap between the performance on training and validation data, contrary to the results found for point classification as illustrated in Figure \ref{fig:act_conv}. Since images are more complex than two or three dimensional points, the models struggle to generalize equally well. Using regularization methods could certainly address this problem, but are not considered in the scope of this work.  

\begin{figure}
    \centering
    \includegraphics[width=6.2cm]{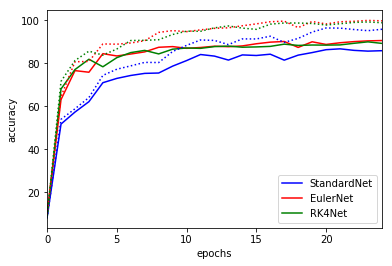}
    \includegraphics[width=6.2cm]{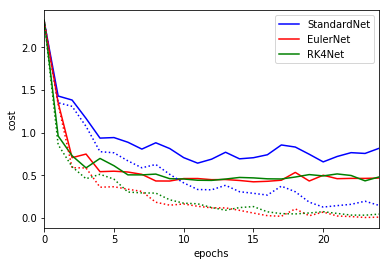}
    \caption[]{Accuracy (left) and cost (right) over the course of epochs on MNIST with network width $\hat{d}=30^2$, depth $L=5$ for \texttt{StandardNet} and $L=100$ for \texttt{EulerNet} and \texttt{RK4Net}, and $\tanh$ activation. Solid lines represent metrics on validation and dotted lines on training data.}
    \label{fig:MNIST_convergence}
\end{figure}

\begin{figure}
    \centering
    \includegraphics[width=6.2cm]{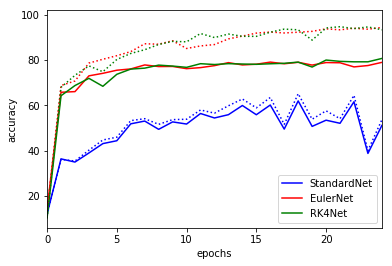}
    \includegraphics[width=6.2cm]{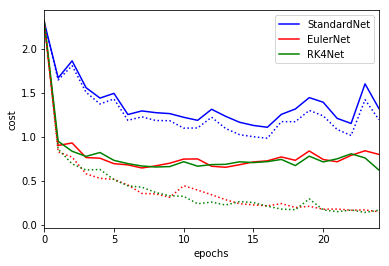}
    \caption[]{Accuracy (left) and cost (right) over the course of epochs on Fashion-MNIST with network width $\hat{d}=30^2$, depth $L=5$ for \texttt{StandardNet} and $L=100$ for \texttt{EulerNet} and \texttt{RK4Net}, and $\tanh$ activation. Solid lines represent metrics on validation and dotted lines on training data.}
    \label{fig:FashionMNIST_convergence}
\end{figure}

In conclusion, the numerical results of point and image classification have many key aspects in common. This suggests that Runge-Kutta networks are a promising architecture to also solve more complex and application-relevant problems.


\section{Outlook}

\label{sec:outlook}

Since the code is written in a modular way and employs the highly flexible tools of PyTorch, its functionality can be easily extended allowing for further experiments. Some examples for continuing the numerical considerations of neural networks based on NODEs or ANODEs include:
(i) testing network models derived by discretizing the ODE by further RK schemes or other integration methods, as in the code supplementing~\cite{chen2018NODE},
(ii) investigating the effect of learning adaptive time step sizes when solving the ODE, implemented as ODENet and ODENet+simplex in \cite{benning2019DLasOCP}, and
(iii) classifying other data types such as videos or audios with networks based on continuous dynamical systems in an augmented space.

\if{
\appendix 

\section{On topological properties of NODEs}

\begin{definition}
Let $  A$ and $ B$ be two sets of points in $\R^n$. Then $ A$ and $B$ are linearly separable if there exist $w \in \R^n$ and $k\in \R$  such that 
\be
w^{\top}x>k\quad \text{for } x \in A,\quad \text{and} \quad w^{\top}x<k \quad \text{for } x \in B,
\ee
\end{definition}

\begin{lemma} Given the sets 
$Q_1=\{(x,y)\in (-1,0)\times (0,1) \}$, $Q_2=\{(x,y)\in (0,1)\times (0,1) \}$, $Q_3=\{(x,y)\in (0,1)\times (0,-1) \}$, and  $Q_4=\{(x,y)\in (-1,0)\times (0,-1) \}$, cf. Figure \ref{ }. 
There does not exist a homeomorphism   $\varphi\colon \cup_{i=1,\dots,4} Q_i \rightarrow \R$ with
\begin{align}\label{hom}
\varphi (x) =\left\{
\begin{array}{ll}
1,& x \in Q_1 \cup Q_3,\\
-1,& x \in Q_2 \cup Q_4.
\end{array}
\right.
\end{align}
\end{lemma}
\begin{proof}
We follow ideas from \cite[Proof of Prop. 2]{Dupont}.
Assume such a $\varphi$ exists.
Let be given points $r_i \in Q_i$, $i=1,\dots, 4$. 
Consider continuous curves $c_1 \subset \bar{Q}_2 \cup \bar{Q}_4$ and $c_2 \subset \bar{Q}_1 \cup \bar{Q}_3$ parameterized with $t\in [0,1]$ with $c_1(0) =r_1$, 
$c_1(1) =r_3$, $c_2(0) =r_2$, $c_2(1) =r_4$.
We denote the parameters for which the curves $c_1$ and $c_2$ intersect in $(0,0)$  by $t_1$ and $t_2$.
Then 
\begin{align}
w^{\top} \varphi(c_1(t))<k\quad \text{for all } t\in [0,1],t\neq t_1,\\ 
w^{\top} \varphi(c_2(t))>k\quad \text{for all } t\in [0,1],t\neq t_2.
\end{align}
By continuity we obtain
\be
1=\varphi(c_1(t_1))=\varphi(c_2(t_2))=-1
\ee
which is a contradiction.
\end{proof}

One can show that the flow given by ordinary differential equations with differentiableXXX nonlinearity is a homeomorphism, see Younes  \cite{XXX}.
Hence, NODES cannot linearly separate the square illustrated in Figure \ref{ }.
}\fi

\providecommand{\bysame}{\leavevmode\hbox to3em{\hrulefill}\thinspace}
\providecommand{\MR}{\relax\ifhmode\unskip\space\fi MR }
\providecommand{\MRhref}[2]{%
  \href{http://www.ams.org/mathscinet-getitem?mr=#1}{#2}
}
\providecommand{\href}[2]{#2}


\begin{thebibliography}{10}

\bibitem{benning2019DLasOCP}
M.~Benning, E.~Celledoni, M.~Ehrhardt, B.~Owren, and C.-B. Schönlieb,
  \emph{Deep learning as optimal control problems: Models and numerical
  methods}, Journal of Computational Dynamics \textbf{6} (2019), 171--198,
  https://doi.org/10.17863/CAM.43231.

\bibitem{MR4308177}
E.~Celledoni, M.~J. Ehrhardt, C.~Etmann, R.~I. McLachlan, B.~Owren, C.-B.
  Sch\"onlieb, and F.~Sherry, \emph{Structure-preserving deep learning},
  European J. Appl. Math. \textbf{32} (2021), no.~5, 888--936. \MR{4308177}

\bibitem{chen2018NODE}
R.~T.~Q. Chen, Y.~Rubanova, J.~Bettencourt, and D.~K. Duvenaud, \emph{Neural
  ordinary differential equations}, Advances in Neural Information Processing
  Systems (S.~Bengio, H.~Wallach, H.~Larochelle, K.~Grauman, N.~Cesa-Bianchi,
  and R.~Garnett, eds.), vol.~31, Curran Associates, Inc., 2018.

\bibitem{BaletZuazua:2021}
E.~Zuazua D.~Ruiz-Balet, \emph{Neural ode control for classification,
  approximation and transport},  (2021), arxiv: 2104.05278.

\bibitem{dupont2019ANODE}
E.~Dupont, A.~Doucet, and Y.W. Teh, \emph{Augmented neural odes}, Adv. Neural
  Inf. Process. Syst. \textbf{32} (2019).

\bibitem{weinan2017MLviaDynamicalSystems}
W.~E, \emph{A proposal on machine learning via dynamical systems},
  Communications in Mathematics and Statistics \textbf{5} (2017), no.~1, 1--11.

\bibitem{ARKN}
E.~Giesecke, \emph{{Augmented-RK-Nets}}, 2021,
  https://github.com/ElisaGiesecke/augmented-RK-Net.

\bibitem{goodfellow2016DL}
I.~Goodfellow, Y.~Bengio, and A.~Courville, \emph{Deep learning}, MIT Press,
  2016.

\bibitem{MR1804658}
W.W. Hager, \emph{Runge-{K}utta methods in optimal control and the transformed
  adjoint system}, Numer. Math. \textbf{87} (2000), no.~2, 247--282.

\bibitem{he2015DResLearningforImage}
K.~He, X.~Zhang, S.~Ren, and J.~Sun, \emph{Deep residual learning for image
  recognition}, Proceedings of the IEEE conference on computer vision and
  pattern recognition, 2016, pp.~770--778.

\bibitem{MR4027841}
C.F. Higham and D.J. Higham, \emph{Deep learning: an introduction for applied
  mathematicians}, SIAM Rev. \textbf{61} (2019), no.~4, 860--891.

\bibitem{MR3881695}
M.~Raissi, P.~Perdikaris, and G.~E. Karniadakis, \emph{Physics-informed neural
  networks: a deep learning framework for solving forward and inverse problems
  involving nonlinear partial differential equations}, J. Comput. Phys.
  \textbf{378} (2019), 686--707.

\bibitem{sanzserna1992symplecticRK}
J.~M. Sanz-Serna, \emph{Symplectic {Runge-Kutta} and related methods: recent
  results}, Physica D \textbf{60} (1992), no.~1–4, 293–302.

\bibitem{sanzserna2015symplecticRKandMore}
J.~M. Sanz-Serna, \emph{Symplectic runge-kutta schemes for adjoint equations,
  automatic differentiation, optimal control and more}, SIAM Review \textbf{58}
  (2015).

\end{thebibliography}
\end{document}